%% file: main.tex
\documentclass[english]{article}
\usepackage[latin9]{inputenc}
\usepackage{array}
\usepackage{textcomp}
\usepackage{url}
\usepackage{multirow}
\usepackage{tablefootnote}
\usepackage{amsmath}
\usepackage{amsthm}
\usepackage{amssymb}
\usepackage{stmaryrd}
\usepackage{graphicx}
\usepackage{wasysym}

\makeatletter

\newcommand{\lyxmathsym}[1]{\ifmmode\begingroup\def\b@ld{bold}
  \text{\ifx\math@version\b@ld\bfseries\fi#1}\endgroup\else#1\fi}

\let\SF@@footnote\footnote
\def\footnote{\ifx\protect\@typeset@protect
    \expandafter\SF@@footnote
  \else
    \expandafter\SF@gobble@opt
  \fi
}
\expandafter\def\csname SF@gobble@opt \endcsname{\@ifnextchar[
  \SF@gobble@twobracket
  \@gobble
}
\edef\SF@gobble@opt{\noexpand\protect
  \expandafter\noexpand\csname SF@gobble@opt \endcsname}
\def\SF@gobble@twobracket[#1]#2{}
\providecommand{\tabularnewline}{\\}

\theoremstyle{plain}
\newtheorem{thm}{\protect\theoremname}
\theoremstyle{plain}
\newtheorem{prop}[thm]{\protect\propositionname}
\theoremstyle{plain}
\newtheorem{lem}[thm]{\protect\lemmaname}

\usepackage{microtype}
\usepackage{graphicx}
\usepackage{subfigure}
\usepackage{booktabs} 
\usepackage{hyperref}
\usepackage[accepted]{icml2020}
\usepackage{amsmath}
\usepackage{etoolbox}
\makeatletter
\patchcmd\@combinedblfloats{\box\@outputbox}{\unvbox\@outputbox}{}{\errmessage{\noexpand patch failed}}
\makeatother
\renewcommand{\cite}{\citep}

\makeatother

\usepackage{babel}
\providecommand{\lemmaname}{Lemma}
\providecommand{\propositionname}{Proposition}
\providecommand{\theoremname}{Theorem}

\begin{document}
\global\long\def\SAM{\mathrm{SAM}}%

\global\long\def\softmax{\mathrm{softmax}}%

\global\long\def\sigmoid{\mathrm{sigmoid}}%

\thispagestyle{empty}
\twocolumn[ 
\icmltitle{Self-Attentive Associative Memory}
\begin{icmlauthorlist} 
\icmlauthor{Hung Le}{to} 
\icmlauthor{Truyen Tran}{to} 
\icmlauthor{Svetha Venkatesh}{to} 
\end{icmlauthorlist}
\icmlaffiliation{to}{Applied AI Institute, Deakin University, Geelong, Australia} 
\icmlkeywords{Machine Learning, ICML}
\icmlcorrespondingauthor{Hung Le}{thai.le@deakin.edu.au}
\vskip 0.3in
]
\printAffiliationsAndNotice{}
\begin{abstract}
Heretofore, neural networks with external memory are restricted to
single memory with lossy representations of memory interactions. A
rich representation of relationships between memory pieces urges a
high-order and segregated relational memory. In this paper, we propose
to separate the storage of individual experiences (item memory) and
their occurring relationships (relational memory). The idea is implemented
through a novel Self-attentive Associative Memory (SAM) operator.
Found upon outer product, SAM forms a set of associative memories
that represent the hypothetical high-order relationships between arbitrary
pairs of memory elements, through which a relational memory is constructed
from an item memory. The two memories are wired into a single sequential
model capable of both memorization and relational reasoning. We achieve
competitive results with our proposed two-memory model in a diversity
of machine learning tasks, from challenging synthetic problems to
practical testbeds such as geometry, graph, reinforcement learning,
and question answering. 
\end{abstract}

\section{Introduction}

\input{intro.tex}

\section{Methods}

\input{method.tex}

\section{Results}

\input{exp.tex}

\section{Related Work}

\input{related.tex}

\section{Conclusions}

\input{discuss.tex}

\subsubsection*{ACKNOWLEDGMENTS}

This research was partially funded by the Australian Government through
the Australian Research Council (ARC). Prof Venkatesh is the recipient
of an ARC Australian Laureate Fellowship (FL170100006).

\bibliographystyle{icml2020}
\bibliography{sam}

\cleardoublepage\newpage{}

\section*{Appendix}

\renewcommand\thesubsection{\Alph{subsection}}

\input{appendix.tex}

\end{document}

%% file: intro.tex
Humans excel in remembering items and the relationship between them
over time \cite{olson2006working,konkel2009relational}. Numerous
neurocognitive studies have revealed this striking ability is largely
attributed to the perirhinal cortex and hippocampus, two brain regions
that support item memory (e.g., objects, events) and relational memory
(e.g., locations of objects, orders of events), respectively \cite{cohen1997memory,buckley2005role}.
Relational memory theory posits that there exists a representation
of critical relationships amongst arbitrary items, which allows inferential
reasoning capacity \cite{eichenbaum1993memory,zeithamova2012hippocampus}.
It remains unclear how the hippocampus can select the stored items
in clever ways to unearth their hidden relationships and form the
relational representation.

Research on artificial intelligence has focused on designing item-based
memory models with recurrent neural networks (RNNs) \cite{hopfield1982neural,elman1990finding,hochreiter1997long}
and memory-augmented neural networks (MANNs) \cite{graves2014neural,graves2016hybrid,le2018variational,le2018learning}.
These memories support long-term retrieval of previously seen items
yet lack explicit mechanisms to represent arbitrary relationships
amongst the constituent pieces of the memories. Recently, further
attempts have been made to foster relational modeling by enabling
memory-memory interactions, which is essential for relational reasoning
tasks \cite{santoro2017simple,santoro2018relational,vaswani2017attention}.
However, no effort has been made to model jointly item memory and
relational memory explicitly.

We argue that dual memories in a single system are crucial for solving
problems that require both memorization and relational reasoning.
Consider graphs wherein each node is associated with versatile features--
as example a road network structure where each node is associated
with diverse features: \emph{graph 1} where the nodes are building
landmarks and \emph{graph 2} where the nodes are flora details. The
goal here is to reason over the structure and output the associated
features of the nodes instead of the pointer or index to the nodes.
Learning to output associated node features enables generalization
to entirely novel features, i.e., a model can be trained to generate
a navigation path with building landmarks (\emph{graph 1}) and tested
in the novel context of generating a navigation path with flora landmarks
(\emph{graph 2}). This may be achieved if the model stores the features
and structures into its item and relational memory, separately, and
reason over the two memories using rules acquired during training. 

Another example requiring both item and relational memory can be understood
by amalgamating the $N^{th}$-farthest \cite{santoro2018relational}
and associative recall \cite{graves2014neural} tasks. $N^{th}$-farthest
requires relational memory to return a fixed one-hot encoding representing
the index to the $N^{th}$-farthest item, while associative recall
returns the item itself, requiring item memory. If these tasks are
amalgamated to compose Relational Associative Recall (RAR) -- return
the $N^{th}$-farthest item from a query (see $\mathsection$ \ref{subsec:Algorithmic-synthetic-tasks}),
it is clear that both item and relational memories are required. 

Three limitations of the current approaches are: $\left(i\right)$
the relational representation is often computed without storing, which
prevents reusing the precomputed relationships in sequential tasks
\cite{vaswani2017attention,santoro2017simple}, $\left(ii\right)$
few works that manage both items and the relationships in a single
memory, make it hard to understand how relational reasoning occurs
\cite{santoro2018relational,schlag2018learning}, $\left(iii\right)$
the memory-memory relationship is coarse since it is represented as
either dot product attention \cite{vaswani2017attention} or weighted
summation via neural networks \cite{santoro2017simple}. Concretely,
the former uses a scalar to measure cosine distance between two vectors
and the later packs all information into one vector via only additive
interactions. 

To overcome the current limitations, we hypothesize a two-memory model,
in which the relational memory exists separately from the item memory.
To maintain a rich representation of the relationship between items,
the relational memory should be higher-order than the item memory.
That is, the relational memory stores multiple relationships, each
of which should be represented by a matrix rather than a scalar or
vector. Otherwise, the capacity of the relational memory is downgraded
to that of the item memory. Finally, as there are two separate memories,
they must communicate to enrich the representation of one another. 

To implement our hypotheses, we introduce a novel operator that facilitates
the communication from the item memory to the relational memory. The
operator, named Self-attentive Associative Memory (SAM) leverages
the dot product attention with our outer product attention. Outer
product is critical for constructing higher-order relational representations
since it retains bit-level interactions between two input vectors,
thus has potential for rich representational learning \cite{smolensky1990tensor}.
SAM transforms a second-order (matrix) item memory into a third-order
relational representation through two steps. First, SAM decodes a
set of patterns from the item memory. Second, SAM associates each
pair of patterns using outer product and sums them up to form a hetero-associative
memory. The memory thus stores relationships between stored items
accumulated across timesteps to form a relational memory.

The role of item memory is to memorize the input data over time. To
selectively encode the input data, the item memory is implemented
as a gated auto-associative memory. Together with previous read-out
values from the relational memory, the item memory is used as the
input for SAM to construct the relational memory. In return, the relational
memory transfers its knowledge to the item memory through a distillation
process. The backward transfer triggers recurrent dynamics between
the two memories, which may be essential for simulating hippocampal
processes \cite{kumaran2012generalization}. Another distillation
process is used to transform the relational memory to the output value. 

Taken together, we contribute a new neural memory model dubbed SAM-based
Two-memory Model (STM) that takes inspiration from the existence of
both item and relational memory in human brain \cite{konkel2009relational}.
In this design, the relational memory is higher-order than the item
memory and thus necessitates a core operator that manages the information
exchange from the item memory to the relational memory. The operator,
namely Self-attentive Associative Memory (SAM), utilizes outer product
to construct a set of hetero-associative memories representing relationships
between arbitrary stored items. We apply our model to a wide range
of tasks that may require both item and relational memory: various
algorithmic learning, geometric and graph reasoning, reinforcement
learning and question-answering tasks. Several analytical studies
on the characteristics of our proposed model are also given in the
Appendix.

%% file: method.tex
\subsection{Outer product attention (OPA)\label{subsec:Outer-product-attention}}

Outer product attention (OPA) is a natural extension of the query-key-value
dot product attention \cite{vaswani2017attention}. Dot product attention
(DPA) for single query $q$ and $n_{kv}$ pairs of key-value can be
formulated as follows,

\begin{equation}
A^{\lyxmathsym{\textdegree}}\ensuremath{\left(q,K,V\right)}=\ensuremath{\sum_{i=1}^{n_{kv}}}\ensuremath{\mathcal{S}\left(q\cdot k_{i}\right)v_{i}}
\end{equation}
where $A^{\lyxmathsym{\textdegree}}\ensuremath{\in\mathbb{R}^{d_{v}}}$,
$q,k_{i}\in\mathbb{R}^{d_{qk}}$, $v_{i}\in\mathbb{R}^{d_{v}}$, $\cdot$
is dot product, and $\mathcal{\mathcal{S}}$ forms $\softmax$ function.
We propose a new outer product attention with similar formulation
yet different meaning,

\begin{equation}
A^{\otimes}\left(q,K,V\right)=\sum_{i=1}^{n_{kv}}\mathcal{F}\left(q\odot k_{i}\right)\otimes v_{i}
\end{equation}
where $A^{\otimes}\in\mathbb{R}^{d_{qk}\times d_{v}}$, $q,k_{i}\in\mathbb{R}^{d_{qk}}$,
$v\in\mathbb{R}^{d_{v}}$, $\odot$ is element-wise multiplication,
$\otimes$ is outer product and $\mathcal{F}$ is chosen as element-wise
$\tanh$ function. 

A crucial difference between DPA and OPA is that while the former
retrieves an attended item $A^{\lyxmathsym{\textdegree}}$, the latter
forms a relational representation $A^{\otimes}$. As a relational
representation, $A^{\otimes}$ captures all bit-level associations
between the key-scaled query and the value. This offers two benefits:
$\left(i\right)$ a higher-order representational capacity that DPA
cannot provide and $\left(ii\right)$ a form of associative memory
that can be later used to retrieve stored item by using a contraction
operation $\mathcal{P}\left(A^{\otimes}\right)$ (see Appendix $\mathsection$
\ref{subsec:OPA-and-SAM}-Prop. \ref{prop:opa_assoc}).

OPA is closely related to DPA. The relationship between the two for
simple $\mathcal{S}$ and $\mathcal{F}$ is presented as follows,
\begin{prop}
\label{prop:opa_g_dpa}Assume that $\mathcal{\mathcal{S}}$ is a linear
transformation: $\mathcal{S}\left(x\right)=ax+b$ ($a,b,x\in\mathbb{R}$),
we can extract $A^{\lyxmathsym{\textdegree}}$ from $A^{\otimes}$
by using an element-wise linear transformation $\mathcal{F}\left(x\right)=a^{f}\odot x+b^{f}$
($a^{f},b^{f},x\in\mathbb{R}^{d_{qk}}$) and a contraction $\mathcal{P}$:
$\mathbb{R}^{d_{qk}\times d_{v}}\rightarrow\mathbb{R}^{d_{v}}$ such
that

\begin{equation}
A^{\lyxmathsym{\textdegree}}\ensuremath{\left(q,K,V\right)}=\ensuremath{\mathcal{P}\left(A^{\otimes}\left(q,K,V\right)\right)}
\end{equation}
\end{prop}

\begin{proof}
see Appendix $\mathsection$ \ref{subsec:Relationship-between-OPA}.
\end{proof}
Moreover, when $n_{kv}=1$, applying a high dimensional transformation
$\mathcal{G}\left(A^{\otimes}\right)$ is equivalent to the well-known
bi-linear model (see Appendix $\mathsection$ \ref{subsec:Relationship-between-OPA-1}-Prop.
\ref{prop:opa_bilinear}). By introducing OPA, we obtain a new building
block that naturally supports both powerful relational bindings and
item memorization. 

\begin{figure*}
\begin{centering}
\includegraphics[width=0.9\textwidth]{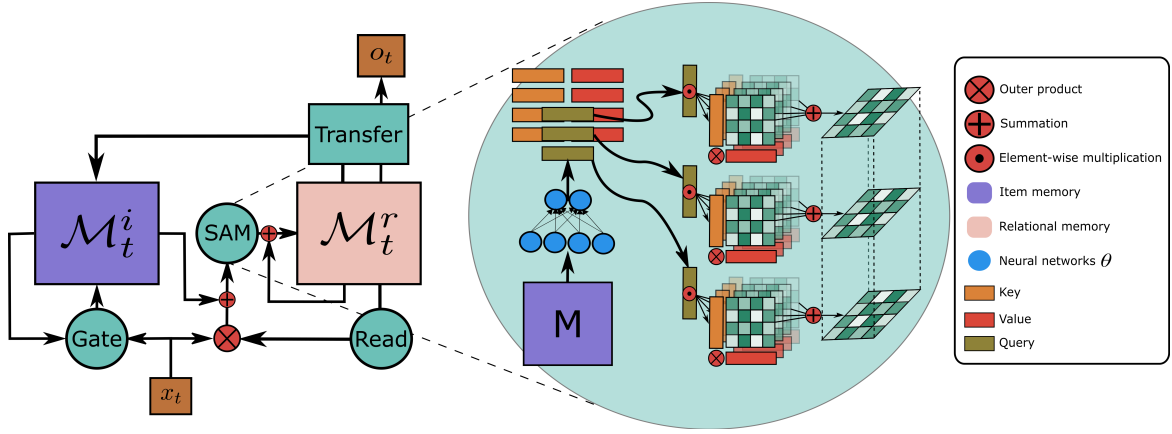}
\par\end{centering}
\caption{STM (left) and SAM (right). SAM uses neural networks $\theta$ to
extract query, key and value elements from a matrix memory $M$. In
this illustration, $n_{q}=3$ and $n_{kv}=4$. Then, it applies outer
product attention to output a $3D$ tensor relational representation.
In STM, at every timestep, the item memory $\mathcal{M}_{t}^{i}$
is updated with new input $x_{t}$ using gating mechanisms (Eq. \ref{eq:gate}).
The item memory plus the read-out from the relational memory is forwarded
to SAM, resulting in a new relational representation to update the
relational memory $\mathcal{M}_{t}^{r}$ (Eq. \ref{eq:Mr_proj}-\ref{eq:mr_update}).
The relational memory transfers its knowledge to the item memory (Eq.
\ref{eq:mr_transfer}) and output value (Eq. \ref{eq:mr_out}). \label{fig:stm}}
\end{figure*}

\subsection{Self-attentive Associative Memory (SAM)\label{subsec:SAM-operator}}

We introduce a novel and generic operator based upon OPA that constructs
relational representations from an item memory. The relational information
is extracted via preserving the outer products between any pairs of
items from the item memory. Hence, we name this operator Self-attentive
Associative Memory (SAM). Given an item memory $M\in\mathbb{R}^{n\times d}$
and parametric weights $\text{\ensuremath{\theta}=}\left\{ W_{q}\in\mathbb{R}^{n_{q}\times n}\right.$,
$W_{k}\in\mathbb{R}^{n_{kv}\times n}$, $\left.W_{v}\in\mathbb{R}^{n_{kv}\times n}\right\} $,
SAM retrieves $n_{q}$ queries, $n_{kv}$ keys and values from $M$
as $M_{q}$, $M_{k}$ and $M_{v}$, respectively,

\begin{equation}
M_{q}=\mathcal{LN}\left(W_{q}M\right)
\end{equation}
\begin{equation}
M_{k}=\mathcal{LN}\left(W_{k}M\right)
\end{equation}
\begin{equation}
M_{v}=\mathcal{LN}\left(W_{v}M\right)
\end{equation}
where $\mathcal{LN}$ is layer normalization operation \cite{ba2016layer}.
Then SAM returns a relational representation $\SAM_{\theta}\left(M\right)$$\in\mathbb{R}^{n_{q}\times d\times d}$,
in which the $s$-th element of the first dimension is defined as 

\begin{align}
\SAM_{\theta}\left(M\right)\left[s\right] & =A^{\otimes}\left(M_{q}\left[s\right],M_{k},M_{v}\right)\\
 & =\sum_{j=1}^{n_{kv}}\mathcal{F}\left(M_{q}\left[s\right]\odot M_{k}\left[j\right]\right)\otimes M_{v}\left[j\right]\label{eq:opsa}
\end{align}
where $s=1,...,n_{q}$. $M_{q}\left[s\right]$, $M_{k}\left[j\right]$
and $M_{v}\left[j\right]$ denote the $s$-th row vector of matrix
$M_{q}$, the $j$-th row vector of matrix $M_{k}$ and $M_{v}$,
respectively. A diagram illustrating SAM operations is given in Fig.
\ref{fig:stm} (right).

It should be noted that $M$ can be any item memory including the
slot-based memories \cite{le2018learning}, direct inputs \cite{vaswani2017attention}
or associative memories \cite{kohonen1972correlation,hopfield1982neural}.
We choose $M\in\mathbb{R}^{d\times d}$ as a form of classical associative
memory, which is biologically plausible \cite{marr1991theory}. Here,
we follow the traditional practice that sets $n=d$ for the associative
item memory. From $M$ we read query, key and value items to form
$\SAM_{\theta}\left(M\right)$--a new set of hetero-associative memories
using Eq. \ref{eq:opsa}. Each hetero-associative memory represents
the relationship between a query and all values. The role of the keys
is to maintain possible perfect retrieval for the item memory (Appendix
$\mathsection$ \ref{subsec:OPA-and-SAM}-Prop. \ref{prop:opa_assoc}). 

The high-order structure of SAM allows it to preserve bit-level relationships
between a query and a value in a matrix. SAM compresses several relationships
with regard to a query by summing all the matrices to form a hetero-associative
memory containing $d^{2}$ scalars, where $d$ is the dimension of
$M$. As there are $n_{kv}$ relationships given 1 query, the summation
results in on average $d^{2}/n_{kv}$ scalars of representation per
relationship, which is greater than $1$ if $d>\sqrt{n_{kv}}$. By
contrast, current self-attention mechanisms use dot product to measure
the relationship between any pair of memory slots, which means 1 scalar
per relationship. 

\subsection{SAM-based Two-Memory Model (STM)\label{subsec:SAM-based-Two-Memory-Model}}

To effectively utilize the SAM operator, we design a system which
consists of two memory units $\mathcal{M}_{t}^{i}\in\mathbb{R}^{d\times d}$
and $\mathcal{M}_{t}^{r}\in\mathbb{R}^{n_{q}\times d\times d}$: one
for items and the other for relationships, respectively. From a high-level
view, at each timestep, we use the current input data $x_{t}$ and
the previous state of memories $\left\{ \mathcal{M}_{t-1}^{i},\mathcal{M}_{t-1}^{r}\right\} $
to produce output $o_{t}$ and new state of memories $\left\{ \mathcal{M}_{t}^{i},\mathcal{M}_{t}^{r}\right\} $.
The memory executions are described as follows. 

\paragraph{$\mathcal{M}^{i}$-Write}

The item memory distributes the data from the input across its rows
in the form of associative memory. For an input $x_{t}$, we update
the item memory as
\begin{align}
X_{t} & =f_{1}\left(x_{t}\right)\otimes f_{2}\left(x_{t}\right)\nonumber \\
\mathcal{M}_{t}^{i} & =\mathcal{M}_{t-1}^{i}+X_{t}\label{eq:hebb_update}
\end{align}
where $f_{1}$ and $f_{2}$ are feed-forward neural networks that
output $d$-dimensional vectors. This update does not discriminate
the input data and inherits the low-capacity of classical associative
memory \cite{rojas2013neural}. We leverage the gating mechanisms
of LSTM \cite{hochreiter1997long} to improve Eq. \ref{eq:hebb_update}
as

\begin{equation}
\mathcal{M}_{t}^{i}=F_{t}\left(\mathcal{M}_{t-1}^{i},x_{t}\right)\odot\mathcal{M}_{t-1}^{i}+I_{t}\left(\mathcal{M}_{t-1}^{i},x_{t}\right)\odot X_{t}\label{eq:gate}
\end{equation}
where $F_{t}$ and $I_{t}$ are forget and input gates, respectively.
Detailed implementation of these gates is in Appendix $\mathsection$
\ref{subsec:Implementation-of-gate}. 

\paragraph{$\mathcal{M}^{r}$-Read }

As relationships stored in \textbf{$\mathcal{M}^{r}$} are represented
as associative memories, the relational memory can be read to reconstruct
previously seen items. As shown in Appendix $\mathsection$ \ref{subsec:OPA-and-SAM}-Prop.
\ref{prop:opa_assoc-1}, the read is basically a two-step contraction,

\begin{equation}
v_{t}^{r}=\softmax\left(f_{3}\left(x_{t}\right)^{\top}\right)\mathcal{M}_{t-1}^{r}f_{2}\left(x_{t}\right)\label{eq:Mr_proj}
\end{equation}
where $f_{3}$ is a feed-forward neural network that outputs a $n_{q}$-dimensional
vector. The read value provides an additional input coming from the
previous state of $\mathcal{M}^{r}$ to relational construction process,
as shown later in Eq. \ref{eq:mr_update}. 

\paragraph{$\mathcal{M}^{i}$-Read $\mathcal{M}^{r}$-Write}

We use SAM to read from $\mathcal{M}^{i}$ and construct a candidate
relational memory, which is simply added to the previous relational
memory to perform the relational update,

\begin{equation}
\mathcal{M}_{t}^{r}=\mathcal{M}_{t-1}^{r}+\alpha_{1}\SAM_{\theta}\left(\mathcal{M}_{t}^{i}+\alpha_{2}v_{t}^{r}\otimes f_{2}\left(x_{t}\right)\right)\label{eq:mr_update}
\end{equation}
where $\alpha_{1}$and $\alpha_{2}$ are blending hyper-parameters.
The input for SAM is a combination of the current item memory $\mathcal{M}_{t}^{i}$
and the association between the extracted item from the previous relational
memory $v_{t}^{r}$ and the current input data $x_{t}$. Here, $v_{t}^{r}$
enhances the relational memory with information from the distant past.
The resulting relational memory stores associations between several
pairs of items in a $3D$ tensors of size $n_{q}\times d\times d$.
In our SAM implementation, $n_{kv}=n_{q}$. 

\paragraph{$\mathcal{M}^{r}$-Transer}

In this phase, the relational knowledge from $\mathcal{M}_{t}^{r}$
is transferred to the item memory by using high dimensional transformation,

\begin{equation}
\mathcal{M}_{t}^{i}=\mathcal{M}_{t}^{i}+\alpha_{3}\mathcal{G}_{1}\circ\mathcal{V}_{f}\circ\mathcal{M}_{t}^{r}\label{eq:mr_transfer}
\end{equation}
where $\mathcal{V}_{f}$ is a function that flattens the first two
dimensions of its input tensor, $\mathcal{G}_{1}$ is a feed-forward
neural network that maps $\mathbb{R}^{\left(n_{q}d\right)\times d}\rightarrow\mathbb{R}^{d\times d}$
and $\alpha_{3}$ is a blending hyper-parameter. As shown in Appendix
$\mathsection$ \ref{subsec:Relationship-between-OPA-1}-Prop. \ref{prop:SAMbili},
with trivial $\mathcal{G}_{1}$, the transfer behaves as if the item
memory is enhanced with long-term stored values from the relational
memory. Hence, $\mathcal{M}^{r}$-Transfer is also helpful in supporting
long-term recall (empirical evidences in $\mathsection$ \ref{subsec:Ablation-study}).
In addition, at each timestep, we distill the relational memory into
an output vector $o_{t}\in\mathbb{R}^{n_{o}}$. We alternatively flatten
and apply high-dimensional transformations as follow,
\begin{equation}
o_{t}=\mathcal{G}_{3}\circ\mathcal{V}_{l}\circ\mathcal{G}_{2}\circ\mathcal{V}_{l}\circ\mathcal{M}_{t}^{r}\label{eq:mr_out}
\end{equation}
where $\mathcal{V}_{l}$ is a function that flattens the last two
dimensions of its input tensor. $\mathcal{G}_{2}$ and $\mathcal{G}_{3}$
are two feed-forward neural networks that map $\mathbb{R}^{n_{q}\times\left(dd\right)}\rightarrow\mathbb{R}^{n_{q}\times n_{r}}$
and $\mathbb{R}^{n_{q}n_{r}}\rightarrow\mathbb{R}^{n_{o}}$, respectively.
$n_{r}$ is a hyper-parameter. 

Unlike the contraction (Eq. \ref{eq:Mr_proj}), the distillation process
does not simply reconstruct the stored items. Rather, thanks to high-dimensional
transformations, it captures bi-linear representations stored in the
relational memory (proof in Appendix $\mathsection$ \ref{subsec:Relationship-between-OPA-1}).
Hence, despite its vector form, the output of our model holds a rich
representation that is useful for both sequential and relational learning.
We discuss further on how to quantify the degree of relational distillation
in Appendix $\mathsection$ \ref{subsec:Order-of-relationship}. The
summary of components of STM is presented in Fig. \ref{fig:stm} (left).

%% file: exp.tex
\subsection{Ablation study\label{subsec:Ablation-study}}

We test different model configurations on two classical tasks for
sequential and relational learning: associative retrieval \cite{ba2016using}
and $N^{th}$-farthest \cite{santoro2018relational} (see Appendix
$\mathsection$ \ref{subsec:Learning-curves-abl} for task details
and learning curves). Our source code is available at \url{https://github.com/thaihungle/SAM}.

\paragraph{Associative retrieval}

This task measures the ability to recall a seen item given its associated
key and thus involves item memory. We use the setting with input sequence
length 30 and 50 \cite{Zhang2017LearningTU}. Three main factors affecting
the item memory of STM are the dimension $d$ of the auto-associative
item memory, the gating mechanisms (Eq. \ref{eq:gate}) and the relational
transfer (Eq. \ref{eq:mr_transfer}). Hence, we ablate our STM ($d=96$,
full features) by creating three other versions: small STM with transfer
($d=48$), small STM without transfer ($d=48$, w/o transfer) and
STM without gates ($d=96$, w/o gates). $n_{q}$ is fixed to $1$
as the task does not require much relational learning.

Table \ref{tab:abl1} reports the number of epochs required to converge
and the final testing accuracy. Without the proposed gating mechanism,
STM struggles to converge, which highlights the importance of extending
the capacity of the auto-associative item memory. The convergence
speed of STM is significantly improved with a bigger item memory size.
Relational transfer seems more useful for longer input sequences since
if requested, it can support long-term retrieval. Compared to other
fast-weight baselines, the full-feature STM performs far better as
it needs only 10 and 20 epochs to solve the tasks of length 30 and
50, respectively. 

\begin{table}
\begin{centering}
\begin{tabular}{ccccc}
\hline 
\multirow{2}{*}{Model} & \multicolumn{2}{c}{Length 30} & \multicolumn{2}{c}{Length 50}\tabularnewline
 & E. & A. & E. & A.\tabularnewline
\hline 
Fast weight$^{*}$  & 50 & 100 & 5000 & 20.8\tabularnewline
WeiNet$^{*}$  & 35 & 100 & 50 & 100\tabularnewline
\hline 
STM ($d=48$, w/o transfer) & 10 & 100 & 100 & 100\tabularnewline
STM ($d=48$) & 20 & 100 & 80 & 100\tabularnewline
STM ($d=96$, w/o gates) & 100 & 24 & 100 & 20\tabularnewline
STM ($d=96$) & \textbf{10} & 100 & \textbf{20} & 100\tabularnewline
\hline 
\end{tabular}
\par\end{centering}
\caption{Comparison of models on associative retrieval task with number of
epochs E. required to converge (lower is better) and convergence test
accuracy A. ($\%$, higher is better). $*$ is reported from \citet{Zhang2017LearningTU}.\label{tab:abl1}}

\end{table}

\paragraph{$\boldsymbol{N^{th}}$-farthest }

This task evaluates the ability to learn the relationship between
stored vectors. The goal is to find the $N^{th}$-farthest vector
from a query vector, which requires a relational memory for distances
between vectors and a sorting mechanism over the distances. For relational
reasoning tasks, the pivot is the number of extracted items $n_{q}$
for establishing the relational memory. Hence, we run our STM with
different $n_{q}=1,4,8$ using the same problem setting (8 $16$-
dimensional input vectors), optimizer (Adam), batch size (1600) as
in \citet{santoro2018relational}. We also run the task with TPR \cite{schlag2018learning}--a
high-order fast-weight model that is designed for reasoning. 

As reported in Table \ref{tab:abl2}, increasing $n_{q}$ gradually
improves the accuracy of STM. As there are 8 input vectors in this
task, literally, at each timestep the model needs to extract 8 items
to compute all pairs of distances. However, as the extracted item
is an entangled representation of all stored vectors and the temporarily
computed distances are stored in separate high-order storage, even
with $n_{q}=1,4$, STM achieves moderate results. With $n_{q}=8$,
STM nearly solves the task perfectly, outperforming RMC by a large
margin. We have tried to tune TPR for this task without success (see
Appendix $\mathsection$ \ref{subsec:Learning-curves-abl}). This
illustrates the challenge of training high-order neural networks in
diverse contexts.

\begin{table}
\begin{centering}
\begin{tabular}{cc}
\hline 
Model & Accuracy ($\%$)\tabularnewline
\hline 
DNC$^{*}$  & 25\tabularnewline
RMC$^{*}$  & 91\tabularnewline
TPR & 13\tabularnewline
\hline 
STM ($n_{q}=1$) & 84\tabularnewline
STM ($n_{q}=4$) & 95\tabularnewline
STM ($n_{q}=8$) & \textbf{98}\tabularnewline
\hline 
\end{tabular}
\par\end{centering}
\caption{Comparison of models on $N^{th}$-farthest task (test accuracy). $*$
is reported from \citet{santoro2018relational}.\label{tab:abl2}}
\end{table}

\subsection{Algorithmic synthetic tasks\label{subsec:Algorithmic-synthetic-tasks}}

\begin{figure*}
\begin{centering}
\includegraphics[width=0.8\textwidth]{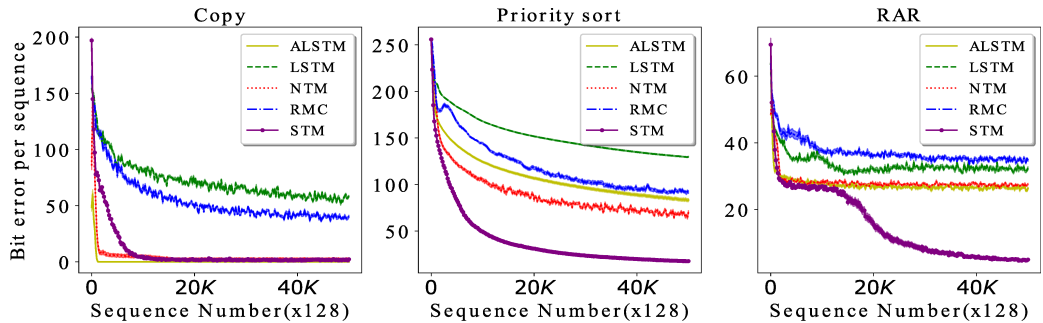}
\par\end{centering}
\caption{Bit error per sequence vs training iteration for algorithmic synthetic
tasks.\label{fig:Learning-curves-on-alg}}
\end{figure*}
Algorithmic synthetic tasks \cite{graves2014neural} examine sequential
models on memorization capacity (eg., Copy, Associative recall) and
simple relational reasoning (eg., Priority sort). Even without explicit
relational memory, MANNs have demonstrated good performance \cite{graves2014neural,Le2020Neural},
but they are verified for only low-dimensional input vectors (\textless 8
bits). As higher-dimensional inputs necessitate higher-fidelity memory
storage, we evaluate the high-fidelity reconstruction capacity of
sequential models for these algorithmic tasks with 32-bit input vectors. 

Two chosen algorithmic tasks are Copy and Priority sort. Item memory
is enough for Copy where the models just output the input vectors
seen in the same order in which they are presented. For Priority sort,
a relational operation that compares the priority of input vectors
is required to produce the seen input vectors in the sorted order
according to the priority score attached to each input vector. The
relationship is between input vectors and thus simply first-order
(see Appendix $\mathsection$ \ref{subsec:Order-of-relationship}
for more on the order of relationship). 

Inspired by Associative recall and $N^{th}$-farthest tasks, we create
a new task named Relational Associative Recall (RAR). In RAR, the
input sequence is a list of items followed by a query item. Each item
is a list of several 32-bit vectors and thus can be interpreted as
a concatenated long vector. The requirement is to reconstruct the
seen item that is farthest or closest (yet unequal) to the query.
The type of the relationship is conditioned on the last bit of the
query vector, i.e., if the last bit is 1, the target is the farthest
and 0 the closest. The evaluated models must compute the distances
from the query item to any other seen items and then compare the distances
to find the farthest/closest one. Hence, this task is similar to the
$N^{th}$-farthest task, which is second-order relational and thus
needs relational memory. However, this task is more challenging since
the models must reconstruct the seen items (32-bit vectors). Compared
to $N=8$ possible one-hot outputs in $N^{th}$-farthest, the output
space in RAR is $2^{32}$ per step, thereby requiring high-fidelity
item memory.

We evaluate our model STM ($n_{q}=8$, $d=96$) with the 4 following
baselines: LSTM \cite{hochreiter1997long}, attentional LSTM \cite{bahdanau2015neural},
NTM \cite{graves2014neural} and RMC \cite{santoro2018relational}.
Details of the implementation are listed in Appendix $\mathsection$
\ref{subsec:Implementation-of-baselines-alg}. The learning curves
(mean and error bar over 5 runs) are presented in Fig. \ref{fig:Learning-curves-on-alg}.

LSTM is often the worst performer as it is based on vector memory.
ALSTM is especially good for Copy as it has a privilege to access
input vectors at every step of decoding. However, when dealing with
relational reasoning, memory-less attention in ALSTM does not help
much. NTM performs well on Copy and moderately on Priority sort, yet
badly on RAR possibly due to its bias towards item memory. Although
equipped with self-attention relational memory, RMC demonstrates trivial
performance on all tasks. This suggests a limitation of using dot-product
attention to represent relationships when the tasks stress memorization
or the relational complexity goes beyond dot-product capacity. Amongst
all models, only the proposed STM demonstrates consistently good performance
where it almost achieves zero errors on these 3 tasks. Notably, for
RAR, only STM can surpass the bottleneck error of 30 bits and reach
$\approx1$ bit error, corresponding to 0\% and 87\% of items perfectly
reconstructed, respectively.

\begin{table*}
\begin{centering}
\begin{tabular}{cccccccc}
\hline 
\multirow{2}{*}{Model} & \multirow{2}{*}{\#Parameters} & \multicolumn{2}{c}{Convex hull} & \multicolumn{2}{c}{TSP} & Shortest  & Minimum \tabularnewline
 &  & $N=5$ & $N=10$ & $N=5$ & $N=10$ & path & spanning tree\tabularnewline
\hline 
LSTM & 4.5 M  & 89.15 & 82.24 & 73.15 (\emph{2.06}) & 62.13 (\emph{3.19}) & 72.38 & 80.11\tabularnewline
ALSTM & 3.7 M & 89.92 & 85.22 & 71.79 (\emph{2.05}) & 55.51 (\emph{3.21}) & 76.70 & 73.40\tabularnewline
DNC & 1.9 M  & 89.42 & 79.47 & 73.24 (\emph{2.05}) & 61.53 (\emph{3.17}) & 83.59 & 82.24\tabularnewline
RMC & 2.8 M & 93.72 & 81.23 & 72.83 (\emph{2.05}) & 37.93 (\emph{3.79}) & 66.71 & 74.98\tabularnewline
\hline 
STM & 1.9 M  & \textbf{96.85} & \textbf{91.88} & \textbf{73.96 (}\textbf{\emph{2.05}}\textbf{)} & \textbf{69.43 (}\textbf{\emph{3.03}}\textbf{)} & \textbf{93.43} & \textbf{94.77}\tabularnewline
\hline 
\end{tabular}
\par\end{centering}
\caption{Prediction accuracy ($\%$) for geometric and graph reasoning with
random one-hot features. Italic numbers are tour length--additional
metric for TSP. Average optimal tour lengths found by brute-force
search for $N=5$ and $10$ are 2.05 and 2.88, respectively.\label{tab:Point-prediction-accuracy}}
\end{table*}
\begin{figure*}
\begin{centering}
\includegraphics[width=1\textwidth]{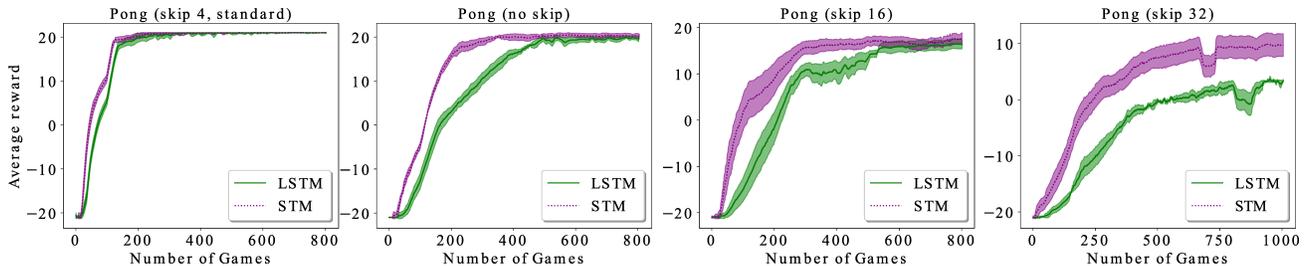}
\par\end{centering}
\caption{Average reward vs number of games for reinforcement learning task
in n-frame skip settings.\label{fig:Learning-curves-on-rl}}
\end{figure*}

\subsection{Geometric and graph reasoning\label{subsec:Geometry-and-graph}}

Problems on geometry and graphs are a good testbed for relational
reasoning, where geometry stipulates spatial relationships between
points, and graphs the relational structure of nodes and edges. Classical
problems include Convex hull, Traveling salesman problem (TSP) for
geometry, and Shortest path, Minimum spanning tree for graph. Convex
hull and TSP data are from \citet{vinyals2015pointer} where input
sequence is a list of points' coordinates (number of points $N\sim\left[5,20\right]$).
Graphs in Shortest path and Minimum spanning tree are generated with
solutions found by Dijkstra and Kruskal algorithms, respectively.
A graph input is represented as a sequence of triplets $\left(node_{1},node_{2},edge_{12}\right)$.
The desired output is a sequence of associated features of the solution
points/nodes (more in Appendix $\mathsection$ \ref{subsec:Geometry-and-graph-1}). 

We generate a random one-hot associated feature for each point/node,
which is stacked into the input vector. This allows us to output the
node's associated features. This is unlike \citet{vinyals2015pointer},
who just outputs the pointers to the nodes. Our modification creates
a challenge for both training and testing. The training is more complex
as the feature of the nodes varies even for the same graph. The testing
is challenging as the associated features are likely to be different
from that in the training. A correct prediction for a timestep is
made when the predicted feature matches perfectly with the ground
truth feature in the timestep. To measure the performance, we use
the average accuracy of prediction across steps. We use the same baselines
as in $\mathsection$ \ref{subsec:Algorithmic-synthetic-tasks} except
that we replace NTM with DNC as DNC performs better on graph reasoning
\cite{graves2016hybrid}. 

We report the best performance of the models on the testing datasets
in Table \ref{tab:Point-prediction-accuracy}. Although our STM has
fewest parameters, it consistently outperforms other baselines by
a significant margin. As usual, LSTM demonstrates an average performance
across tasks. RMC and ALSTM are only good at Convex hull. DNC performs
better on graph-like problems such as Shortest path and Minimum spanning
tree. For the NP-hard TSP ($N=5$), despite moderate point accuracy,
all models achieve nearly minimal solutions with an average tour length
of $2.05$. When increasing the difficulty with more points ($N=10$),
none of these models reach an average optimal tour length of $2.88$.
However, only STM approaches closer to the optimal solution without
the need for pointer and beam search mechanisms. Armed with both item
and relational memory, STM's superior performance suggests a qualitative
difference in the way STM and other methods solve these problems.

\subsection{Reinforcement learning}

Memory is helpful for partially observable Markov decision process
\cite{bakker2002reinforcement}. We apply our memory to LSTM agents
in Atari game environment using A3C training \cite{mnih2016asynchronous}.
More details are given in Appendix $\mathsection$ \ref{subsec:Reinforcement-learning-task}.
In Atari games, each state is represented as the visual features of
a video frame and thus is partially observable. To perform well, RL
agents should remember and relate several frames to model the game
state comprehensively. These abilities are challenged when over-sampling
and under-sampling the observation, respectively. We analyze the performance
of LSTM agents and their STM-augmented counterparts under these settings
using a game: Pong. 

To be specific, we test the two agents on different frame skips (0,
4, 16, 32). We create $n$-frame skip setting by allowing the agent
to see the environment only after every $n$ frames, where 4-frame
skip is standard in most Atari environments. When no frameskip is
applied (over-sampling), the number of observations is dense and the
game is long (up to 9000 steps per game), which requires high-capacity
item memory. On the contrary, when a lot of frames are skipped (under-sampling),
the observations become scarce and the agents must model the connection
between frames meticulously, demanding better relational memory. 

We run each configuration 5 times and report the mean and error bar
of moving average reward (window size $=100$) through training time
in Fig. \ref{fig:Learning-curves-on-rl}. In a standard condition
(4-frame skip), both baselines can achieve perfect performance and
STM outperforms LSTM slightly in terms of convergence speed. The performance
gain becomes clearer under extreme conditions with over-sampling and
under-sampling. STM agents require fewer practices to accomplish higher
rewards, especially in the 32-frame skip environment, which illustrates
that having strong item and relational memory in a single model is
beneficial to RL agents. 

\subsection{Question answering }

bAbI is a question answering dataset that evaluates the ability to
remember and reason on textual information \cite{weston2015towards}.
Although synthetically generated, the dataset contains 20 challenging
tasks such as pathfinding and basic induction, which possibly require
both item and relational memory. Following \citet{schlag2018learning},
each story is preprocessed into a sentence-level sequence, which is
fed into our STM as the input sequence. We jointly train STM for all
tasks using normal supervised training (more in Appendix $\mathsection$
\ref{subsec:babi_task}). We compare our model with recent memory
networks and report the results in Table \ref{tab:Results-for-bAbI}. 

MANNs such as DNC and NUTM have strong item memory, yet do not explicitly
support relational learning, leading to significantly higher errors
compared to other models. On the contrary, TPR is explicitly equipped
with relational bindings but lack of item memory and thus clearly
underperforms our STM. Universal Transformer (UT) supports a manually
set item memory with dot product attention, showing higher mean error
than STM with learned item memory and outer product attention. Moreover,
our STM using normal supervised loss outperforms MNM-p trained with
meta-level loss, establishing new state-of-the-arts on bAbI dataset.
Notably, STM achieves this result with low variance, solving 20 tasks
for 9/10 run (see Appendix $\mathsection$ \ref{subsec:babi_task}). 

\begin{table}
\begin{centering}
\begin{tabular}{lcc}
\hline 
\multirow{2}{*}{{\small{}Model}} & \multicolumn{2}{c}{{\small{}Error}}\tabularnewline
 & {\small{}Mean } & {\small{}Best}\tabularnewline
\hline 
{\small{}DNC \cite{graves2016hybrid}} & {\small{}12.8 $\pm$ 4.7} & {\small{}3.8}\tabularnewline
{\small{}NUTM \cite{Le2020Neural} } & {\small{}5.6 $\pm$ 1.9} & {\small{}3.3}\tabularnewline
{\small{}TPR \cite{schlag2018learning}} & {\small{}1.34 $\pm$ 0.52 } & {\small{}0.81}\tabularnewline
{\small{}UT \cite{dehghani2018universal}} & {\small{}1.12 $\pm$ 1.62} & {\small{}0.21}\tabularnewline
{\small{}MNM-p \cite{munkhdalai2019metalearned}} & {\small{}0.55 $\pm$ 0.74} & {\small{}0.18}\tabularnewline
\hline 
{\small{}STM } & \textbf{\small{}0.39 $\pm$ 0.18} & \textbf{\small{}0.15}\tabularnewline
\hline 
\end{tabular}
\par\end{centering}
\caption{bAbI task: mean $\pm$ std. and best error over 10 runs.\label{tab:Results-for-bAbI}}
\end{table}

%% file: related.tex
\paragraph{Background on associative memory}

Associative memory is a classical concept to model memory in the brain
\cite{marr1991theory}. While outer product is one common way to form
the associative memory, different models employ different memory retrieval
mechanisms. For example, Correlation Matrix Memory (CMM) and Hopfield
network use dot product and recurrent networks, respectively \cite{kohonen1972correlation,hopfield1982neural}.
The distinction between our model and other associative memories lies
in the fact that our model's association comes from several pieces
of the memory itself rather than the input data. Also, unlike other
two-memory systems \cite{Le:2018:DMN:3219819.3219981,Le2020Neural}
that simulate data/program memory in computer architecture, our STM
resembles item and relational memory in human cognition.

\paragraph{Background on attention}

Attention is a mechanism that allows interactions between a query
and a set of stored keys/values \cite{graves2014neural,bahdanau2014neural}.
Self-attention mechanism allows stored items to interact with each
other either in forms of feed-forward \cite{vaswani2017attention}
or recurrent \cite{santoro2018relational,le2018learning} networks.
Modeling memory interactions can also be achieved via attention over
a set of parallel RNNs \cite{henaff2016tracking}. Although some form
of relational memory can be kept in these approaches, they all use
dot product attention to measure interactions per attention head as
a scalar, and thus loose much relational information. We use outer
product to represent the interactions as a matrix and thus our outer
product self-attention is supposed to be richer than the current self-attention
mechanisms (Prop. \ref{prop:opa_g_dpa}).

\paragraph{SAM as fast-weight}

Outer product represents Hebbian learning--a fast learning rule that
can be used to build fast-weights \cite{cogprints1380}. As the name
implies, fast-weights update whenever an input is introduced to the
network and stores the input pattern temporarily for sequential processing
\cite{ba2016using}. Meta-trained fast-weights \cite{munkhdalai2019metalearned}
and gating of fast-weights \cite{schlag2017gated,Zhang2017LearningTU}
are introduced to improve memory capacity. Unlike these fast-weight
approaches, our model is not built on top of other RNNs. Recurrency
is naturally supported within STM.

The tensor product representation (TPR), which is a form of high-order
fast-weight, can be designed for structural reasoning \cite{smolensky1990tensor}.
In a recent work \cite{schlag2018learning}, a third-order TPR resembles
our relational memory $\mathcal{M}_{t}^{r}$ where both are $3D$
tensors. However, TPR does not enable interactions amongst stored
patterns through self-attention mechanism. The meaning of each dimension
of the TPR is not related to that of $\mathcal{M}_{t}^{r}$. More
importantly, TPR is restricted to question answering task. 

\paragraph{SAM as bi-linear model}

Bi-linear pooling produces output from two input vectors by considering
all pairwise bit interactions and thus can be implemented by means
of outer product \cite{tenenbaum2000separating}. To reduce computation
cost, either low-rank factorization \cite{yu2017multi} or outer product
approximation \cite{pham2013fast} is used. These approaches aim to
enrich feed-forward layers with bi-linear poolings yet have not focused
on maintaining a rich memory of relationships. 

Low-rank bi-linear pooling is extended to perform visual attentions
\cite{kim2018bilinear}. It results in different formulation from
our outer product attention, which is equivalent to full rank bi-linear
pooling ($\mathsection$ \ref{subsec:Outer-product-attention}). These
methods are designed for static visual question answering while our
approach is used to maintain a relational memory over time, which
can be applied to any sequential problem.

%% file: discuss.tex
We have introduced the SAM-based Two-memory Model (STM) that implements
both item and relational memory. To wire up the two memory system,
we employ a novel operator named Self-attentive Associative Memory
(SAM) that constructs the relational memory from outer-product relationships
between arbitrary pieces of the item memory. We apply read, write
and transfer operators to access, update and distill the knowledge
from the two memories. The ability to remember items and their relationships
of the proposed STM is validated through a suite of diverse tasks
including associative retrieval, $N^{th}$-farthest, vector algorithms,
geometric and graph reasoning, reinforcement learning and question
answering. In all scenarios, our model demonstrates strong performance,
confirming the usefulness of having both item and relational memory
in one model.

%% file: appendix.tex
\subsection{Relationship between OPA and DPA\label{subsec:Relationship-between-OPA}}

\begin{table*}
\begin{centering}
\begin{tabular}{cccc}
\hline 
Model & Addition complexity & Multiplication complexity & Physical storage for relationships\tabularnewline
\hline 
DPA & $O\left(\left(d_{qk}n_{q}+d_{v}\right)n_{kv}\right)$ & $O\left(\left(d_{qk}+d_{v}\right)n_{q}n_{kv}\right)$ & $O\left(n_{q}n_{kv}\right)$\tabularnewline
OPA & $O\left(n_{q}n_{kv}d_{qk}d_{v}\right)$ & $O\left(n_{q}d_{qk}d_{v}\right)$ & $O\left(n_{q}d_{qk}d_{v}\right)$\tabularnewline
\hline 
\end{tabular}
\par\end{centering}
\caption{Computational complexity of DPA and OPA with $n_{q}$ queries and
$n_{kv}$ key-value pairs. $d_{qk}$ denotes query or key size, while
$d_{v}$ value size. \label{tab:O_opa}}
\end{table*}
\begin{table}
\begin{centering}
\begin{tabular}{cc}
\hline 
Model & Wall-clock time (second)\tabularnewline
\hline 
LSTM & 0.1\tabularnewline
NTM & 1.8\tabularnewline
RMC & 0.3\tabularnewline
\hline 
STM & 0.3\tabularnewline
\hline 
\end{tabular}
\par\end{centering}
\caption{Wall-clock time to process a batch of data on Priority Sort task.
The batch size is 128. All models are implemented using Pytorch, have
around 1 million parameters and run on the same machine with Tesla
V100-SXM2 GPU. \label{tab:wctime}}
\end{table}
\begin{lem}
\label{lem:induct}For $\forall n_{i},n_{j}\in\mathbb{N^{+}}$, 

\begin{equation}
\sum_{i=1}^{n_{i}}\sum_{j=1}^{n_{j}}q_{j}k_{ij}v_{i}=\sum_{j=1}^{n_{j}}\sum_{i=1}^{n_{i}}q_{j}k_{ij}v_{i}\label{eq:lem1_eq}
\end{equation}
where $q_{j},k_{ij},v_{i}\in\mathbb{R}$.
\end{lem}

\begin{proof}
We will prove by induction for all $n_{j}\in\mathbb{N^{+}}$. 

Base case: when $n_{j}=1$, the $LHS=RHS=\sum_{i}^{n_{i}}q_{1}k_{i1}v_{i}$.
Let $t\in\mathbb{N}^{+}$ be given and suppose Eq. \ref{eq:lem1_eq}
is true for $n_{j}=t$. Then

\begin{align*}
\sum_{i=1}^{n_{i}}\sum_{j=1}^{t+1}q_{j}k_{ij}v_{i} & =\sum_{i=1}^{n_{i}}\left(q_{t+1}k_{it+1}v_{i}+\sum_{j=1}^{t}q_{j}k_{ij}v_{i}\right)\\
 & =\sum_{i=1}^{n_{i}}q_{t+1}k_{it+1}v_{i}+\sum_{i=1}^{n_{i}}\sum_{j=1}^{t}q_{j}k_{ij}v_{i}\\
 & =\sum_{i=1}^{n_{i}}q_{t+1}k_{it+1}v_{i}+\sum_{j=1}^{t}\sum_{i=1}^{n_{i}}q_{j}k_{ij}v_{i}\\
 & =\sum_{j=1}^{t+1}\sum_{i=1}^{n_{i}}q_{j}k_{ij}v_{i}
\end{align*}
Thus, Eq. \ref{eq:lem1_eq} holds for $n_{j}=t+1$ and $\forall n_{j}\in\mathbb{N^{+}}$
by the principle of induction. 
\end{proof}
\begin{prop}
\label{prop:opa_g_dpa1}Assume that $\mathcal{\mathcal{S}}$ is a
linear transformation: $\mathcal{\mathcal{\mathcal{S}}}\left(x\right)=ax+b$
($a,b,x\in\mathbb{R}$), we can extract $A^{\lyxmathsym{\textdegree}}$
from $A^{\otimes}$ by using an element-wise linear transformation
$\mathcal{F}\left(x\right)=a^{f}\odot x+b^{f}$ ($a^{f},b^{f},x\in\mathbb{R}^{d_{qk}}$)
and a contraction $\mathcal{P}$: $\mathbb{R}^{d_{qk}\times d_{v}}\rightarrow\mathbb{R}^{d_{v}}$
such that

\begin{equation}
A^{\lyxmathsym{\textdegree}}\ensuremath{\left(q,K,V\right)}=\ensuremath{\mathcal{P}\left(A^{\otimes}\left(q,K,V\right)\right)}
\end{equation}
where
\end{prop}

\begin{equation}
A^{\lyxmathsym{\textdegree}}\ensuremath{\left(q,K,V\right)}=\ensuremath{\sum_{i=1}^{n_{kv}}}\ensuremath{\mathcal{S}\left(q\cdot k_{i}\right)v_{i}}
\end{equation}

\begin{equation}
A^{\otimes}\left(q,K,V\right)=\sum_{i=1}^{n_{kv}}\mathcal{F}\left(q\odot k_{i}\right)\otimes v_{i}
\end{equation}

\begin{proof}
We derive the LHS. Let $u_{i}$ denote the scalar $\mathcal{S}\left(q\cdot k_{i}\right)$,
then

\begin{align*}
u_{i} & =\mathcal{S}\left(q\cdot k_{i}\right)=\mathcal{S}\left(\sum_{j=1}^{d_{qk}}q_{j}k_{ij}\right)\\
 & =\sum_{j=1}^{d_{qk}}aq_{j}k_{ij}+b
\end{align*}
where $q_{j}$ and $k_{ij}$ are the $j$-th elements of vector $q$
and $k_{i}$, respectively. Let $l\in\mathbb{R}^{d_{v}}$ denote the
vector $A^{\lyxmathsym{\textdegree}}\ensuremath{\left(q,K,V\right)}=\ensuremath{\sum_{i=1}^{n_{kv}}u_{i}v_{i}}$,
then the $t$-th element of $l$ is 

\begin{align}
l_{t} & =\sum_{i=1}^{n_{kv}}u_{i}v_{it}\nonumber \\
 & =\sum_{i=1}^{n_{kv}}\left(\sum_{j=1}^{d_{qk}}aq_{j}k_{ij}+b\right)v_{it}\nonumber \\
 & =\sum_{i=1}^{n_{kv}}\sum_{j=1}^{d_{qk}}aq_{j}k_{ij}v_{it}+b\sum_{i=1}^{n_{kv}}v_{it}\nonumber \\
 & =a\sum_{i=1}^{n_{kv}}\sum_{j=1}^{d_{qk}}q_{j}k_{ij}v_{it}+b\sum_{i=1}^{n_{kv}}v_{it}
\end{align}
We derive the RHS. Let $d_{i}$ denote the vector $\mathcal{F}\left(q\odot k_{i}\right)$,
then the $j$-th element of $d_{i}$ is 

\begin{align}
d_{ij} & =\mathcal{F}\left(q_{j}k_{ij}\right)\nonumber \\
 & =a_{j}^{f}q_{j}k_{ij}+b_{j}^{f}
\end{align}
Let $e\in\mathbb{R}^{d_{qk}\times d_{v}}$ denote the matrix $A^{\otimes}\left(q,K,V\right)=\sum_{i=1}^{n_{kv}}d_{i}\otimes v_{i}$,
then the $j$-th row, $t$-column element of $e$ is

\begin{align}
e_{jt} & =\sum_{i=1}^{n_{kv}}d_{ij}v_{it}\nonumber \\
 & =\sum_{i=1}^{n_{kv}}\left(a_{j}^{f}q_{j}k_{ij}+b_{j}^{f}\right)v_{it}\nonumber \\
 & =\sum_{i=1}^{n_{kv}}a_{j}^{f}q_{j}k_{ij}v_{it}+b_{j}^{f}\sum_{i=1}^{n_{kv}}v_{it}
\end{align}
Let $r\in\mathbb{R}^{d_{v}}$ denote the vector $\sum_{j=1}^{d_{qk}}e_{j}$,
then the $t$-th element of $r$ is 

\begin{align}
r_{t} & =\sum_{j=1}^{d_{qk}}e_{jt}\nonumber \\
 & =\sum_{j=1}^{d_{qk}}\left(\sum_{i=1}^{n_{kv}}a_{j}^{f}q_{j}k_{ij}v_{it}+b_{j}^{f}\sum_{i=1}^{n_{kv}}v_{it}\right)\nonumber \\
 & =\sum_{j=1}^{d_{qk}}\sum_{i=1}^{n_{kv}}a_{j}^{f}q_{j}k_{ij}v_{it}+\sum_{j=1}^{d_{qk}}b_{j}^{f}\sum_{i=1}^{n_{kv}}v_{it}\label{eq:lem1_rhs}
\end{align}
We can always choose $a_{j}^{f}=a$ and $\sum_{j=1}^{d_{qk}}b_{j}^{f}=b$.
Eq. \ref{eq:lem1_rhs} becomes,

\[
r_{t}=a\sum_{j=1}^{d_{qk}}\sum_{i=1}^{n_{kv}}q_{j}k_{ij}v_{it}+b\sum_{i}^{n_{kv}}v_{it}
\]
According to Lemma \ref{lem:induct}, $l_{t}=r_{t}$ $\forall d_{qk},n_{kv}\in\mathbb{N^{+}}\Rightarrow l=r$.
Also, $\exists\mathcal{P}$ as a contraction: $\mathcal{P}\left(X\right)=a_{p}X$
with $a_{p}=\left[1,...,1\right]\in\mathbb{R}^{1\times d_{qk}}$.
\end{proof}

We compare the complexity of DPA and OPA in Table \ref{tab:O_opa}.
In general, compared to that of DPA, OPA's complexity is increased
by an order of magnitude, which is equivalent to the size of the patterns.
In practice, we keep that value small (96) to make the training efficient.
That said, due to its high-order nature, our memory model still maintains
enormous memory space. In terms of speed, STM's running time is almost
the same as RMC's and much faster than that of DNC or NTM. Table \ref{tab:wctime}
compares the real running time of several memory-based models on Priority
Sort task.

\subsection{Relationship between OPA and bi-linear model\label{subsec:Relationship-between-OPA-1}}
\begin{prop}
\label{prop:opa_bilinear}Given the number of key-value pairs $n_{kv}=1$,
and $\mathcal{G}$ is a high dimensional linear transformation $\mathcal{G}:\mathbb{R}^{d_{qk}\times d_{v}}\rightarrow\mathbb{R}^{n}$,
$\mathcal{G}\left(X\right)=W^{g}\mathcal{V}\left(X\right)$ where
$W^{g}\in\mathbb{R}^{n\times d_{qk}d_{v}}$, $\mathcal{V}$ is a function
that flattens its input tensor, then $\mathcal{G}\left(A^{\otimes}\left(q,K,V\right)\right)$
can be interpreted as a bi-linear model between $f$ and $v_{1}$,
that is

\begin{equation}
\mathcal{G}\left(A^{\otimes}\left(q,K,V\right)\right)\left[s\right]=\sum_{j=1}^{d_{qk}}\sum_{t=1}^{d_{v}}W^{g}\left[s,j,t\right]f\left[j\right]v_{1}\left[t\right]\label{eq:opa_bil}
\end{equation}
where $W^{g}\left[s,j,t\right]=W^{g}\left[s\right]\left[\left(j-1\right)d_{v}+t\right]$,$s=1,...,n$,
$j=1,...,d_{qk}$, $t=1,...,d_{v}$, and $f=\mathcal{F}\left(q\odot k_{1}\right)$.
\end{prop}

\begin{proof}
By definition, 

\begin{align*}
\mathcal{V}\left(\mathcal{F}\left(q\odot k_{1}\right)\otimes v_{1}\right)\left[\left(j-1\right)d_{v}+t\right] & =\left(\mathcal{F}\left(q\odot k_{1}\right)\otimes v_{1}\right)\left[j\right]\left[t\right]\\
 & =\mathcal{F}\left(q\odot k_{1}\right)\left[j\right]v_{1}\left[t\right]
\end{align*}
We derive the LHS,

\begin{align*}
\mathcal{G}\left(A^{\otimes}\left(q,K,V\right)\right)\left[s\right] & =\left(W^{g}\mathcal{V}\left(\mathcal{F}\left(q\odot k_{1}\right)\otimes v_{1}\right)\right)\left[s\right]\\
 & =\sum_{u=1}^{d_{qk}d_{v}}W^{g}\left[s\right]\left[u\right]\mathcal{V}\left(\mathcal{F}\left(q\odot k_{1}\right)\otimes v_{1}\right)\left[u\right]\\
 & =\sum_{\left(j-1\right)d_{v}+t}^{d_{qk}d_{v}}\left(W^{g}\left[s\right]\left[\left(j-1\right)d_{v}+t\right]\right.\\
 & \times\left.\mathcal{V}\left(\mathcal{F}\left(q\odot k_{1}\right)\otimes v_{1}\right)\left[\left(j-1\right)d_{v}+t\right]\right)\\
 & =\sum_{j=1}^{d_{qk}}\sum_{t=1}^{d_{v}}W^{g}\left[s,j,t\right]\mathcal{F}\left(q\odot k_{1}\right)\left[j\right]v_{1}\left[t\right]
\end{align*}
which equals the RHS.
\end{proof}
Prop. \ref{prop:opa_bilinear} is useful since it demonstrates the
representational capacity of OPA is at least equivalent to bi-linear
pooling, which is richer than low-rank bi-linear pooling using Hadamard
product, or bi-linear pooling using identity matrix of the bi-linear
form (dot product), or the vanilla linear models using traditional
neural networks. 
\begin{prop}
\label{prop:SAMbili}Given the number of queries $n_{q}=d_{qk}$,
the number of key-value pairs $n_{kv}=1$, $\mathcal{M}_{t}^{r}=\SAM_{\theta}\left(M\right)$
where $M$ is an instance of the item memory in the past, and $\mathcal{G}$
is a high dimensional linear transformation $\mathcal{G}:\mathbb{R}^{n_{q}\times d_{qk}\times d_{v}}\rightarrow\mathbb{R}^{d_{qk}\times d_{v}}$,
$\mathcal{G}\left(X\right)=W^{g}\mathcal{V}_{f}\left(X\right)$ where
$W^{g}\in\mathbb{R}^{d_{qk}\times n_{q}d_{qk}}$, $\mathcal{V}_{f}$
is a function that flattens the first two dimensions of its input
tensor, then Eq. \ref{eq:mr_transfer} can be interpreted as a Hebbian
update to the item memory.
\end{prop}

\begin{proof}
Let $k_{1}=M_{k}$ and $v_{1}=M_{v}$ when $n_{kv}=1$, by definition
$\mathcal{V}_{f}\left(\SAM_{\theta}\left(M\right)\right)\left[\left(s-1\right)d_{qk}+j,t\right]=\mathcal{F}\left(M_{q}\left[s\right]\odot k_{1}\right)\left[j\right]v_{1}\left[t\right]$.
We derive,

\begin{align}
\mathcal{G}\left(\SAM_{\theta}\left(M\right)\right)\left[i,t\right] & =\left(W^{g}\mathcal{V}_{f}\left(\SAM_{\theta}\left(M\right)\right)\right)\text{\ensuremath{\left[i,t\right]}}\nonumber \\
 & =\sum_{u=1}^{n_{q}d_{qk}}W^{g}\left[i,u\right]\mathcal{V}_{f}\left(\SAM_{\theta}\left(M\right)\right)\left[u,t\right]\nonumber \\
 & =\sum_{\left(s-1\right)d_{qk}+j=1}^{n_{q}d_{qk}}\left(W^{g}\left[i,\left(s-1\right)d_{qk}+j\right]\right.\nonumber \\
 & \times\left.\mathcal{F}\left(M_{q}\left[s\right]\odot k_{1}\right)\left[j\right]v_{1}\left[t\right]\right)\nonumber \\
 & =\sum_{s=1}^{n_{q}}\sum_{j=1}^{d_{qk}}W^{g}\left[i,s,j\right]f\left[s,j\right]v_{1}\left[t\right]\label{eq:sam_bil}
\end{align}
where $f\left[s,j\right]=\mathcal{F}\left(M_{q}\left[s\right]\odot k_{1}\right)\left[j\right]=\mathcal{F}\left(M_{q}\left[s,j\right]k_{1}\left[j\right]\right)$.
It should be noted that with trivial rank-one $W^{g}$: $W^{g}\left[i\right]=d_{i}\mathcal{V}_{f}\left(I\right)$,
$d_{i}\in\mathbb{R}$, $I$ is the identity matrix, Eq. \ref{eq:sam_bil}
becomes 

\begin{align*}
\mathcal{G}\left(\SAM_{\theta}\left(M\right)\right)\left[i,t\right] & =d\left[i\right]v_{1}\left[t\right]\\
\mathcal{\Rightarrow G}\left(\SAM_{\theta}\left(M\right)\right) & =d\otimes v_{1}
\end{align*}
where $d\in\mathbb{R}^{d_{qk}},d\left[i\right]=d_{i}\sum_{s=1}^{n_{q}}\mathcal{F}\left(M_{q}\left[s,s\right]k_{1}\left[s\right]\right)$.
Eq. \ref{eq:mr_transfer} reads

\[
\mathcal{M}_{t}^{i}=\mathcal{M}_{t}^{i}+\alpha_{3}d\otimes v_{1}
\]
which is a Hebbian update with the updated value $v_{1}$. As $v_{1}$
is a stored pattern extracted from $M$ encoded in the relational
memory, the item memory is enhanced with a long-term stored value
from the relational memory.
\end{proof}

\subsection{OPA and SAM as associative memory\label{subsec:OPA-and-SAM}\protect\footnote{In this section, we use these following properties without explanation:
$a^{\top}\left(b\otimes c\right)=\left(a^{\top}b\right)c^{\top}$
and $\left(b\otimes c\right)a=\left(c^{\top}a\right)b.$ }}
\begin{prop}
\label{prop:opa_assoc}If $\mathcal{P}$ is a contraction: $\mathbb{R}^{d_{qk}\times d_{v}}\rightarrow\mathbb{R}^{d_{v}}$,
$\mathcal{P}\left(X\right)=a_{p}X,a_{p}\in\mathbb{R}^{1\times d_{qk}}$,
then $A^{\otimes}\left(q,K,V\right)$ is an associative memory that
stores patterns $\left\{ v_{i}\right\} _{i=1}^{n_{kv}}$ and $\mathcal{P}\left(A^{\otimes}\left(q,K,V\right)\right)$
is a retrieval process. Perfect retrieval is possible under the following
three conditions,

$\left(1\right)\left\{ k_{i}\right\} _{i=1}^{n_{kv}}$ form a set
of linearly independent vectors

$\left(2\right)q_{i}\neq0$, $i=1,...,d_{qk}$ 

$\left(3\right)\mathcal{F}$ is chosen as $\mathcal{F}\left(x\right)=a^{f}\odot x$
($a^{f},x\in\mathbb{R}^{d_{qk}}$, $a_{i}^{f}\neq0$, $i=1,...,d_{qk}$)
\end{prop}

\begin{proof}
By definition, $A^{\otimes}\left(q,K,V\right)$ forms a hetero-associative
memory between $x_{i}=\mathcal{F}\left(q\odot k_{i}\right)$ and $v_{i}$.
If $\left\{ x_{i}\right\} _{i=1}^{n_{kv}}$ are orthogonal, given
some $\mathcal{P}$ with $a_{p}=\frac{x_{j}^{\top}}{\left\Vert x_{j}^{\top}\right\Vert }$,
then 

\begin{align*}
\mathcal{P}\left(A^{\otimes}\left(q,K,V\right)\right) & =\frac{x_{j}^{\top}}{\left\Vert x_{j}^{\top}\right\Vert }\sum_{i=1}^{n_{kv}}x_{i}\otimes v_{i}\\
 & =\sum_{i=1,i\neq j}^{n_{kv}}\frac{\left(x_{j}^{\top}x_{i}\right)}{\left\Vert x_{j}^{\top}\right\Vert }v_{i}^{\top}+\frac{\left(x_{j}^{\top}x_{j}\right)}{\left\Vert x_{j}^{\top}\right\Vert }v_{j}^{\top}\\
 & =v_{j}^{\top}
\end{align*}
Hence, we can perfectly retrieve some stored pattern $v_{j}$ using
its associated $\mathcal{P}$. In practice, linearly independent $\left\{ x_{i}\right\} _{i=1}^{n_{kv}}$
is enough for perfect retrieval since we can apply Gram--Schmidt
process to construct orthogonal $\left\{ x_{i}\right\} _{i=1}^{n_{kv}}$.
Another solution is to follow Widrow-Hoff incremental update

\begin{align*}
A^{\otimes}\left(q,K,V\right)\left(0\right) & =0\\
A^{\otimes}\left(q,K,V\right)\left(i\right) & =A^{\otimes}\left(q,K,V\right)\left(i-1\right)\\
 & +\left(v_{i}-A^{\otimes}\left(q,K,V\right)\left(i-1\right)x_{i}\right)\otimes x_{i}
\end{align*}
which also results in possible perfect retrieval given $\left\{ x_{i}\right\} _{i=1}^{n_{kv}}$
are linearly independent.

Now, we show that if $\left(1\right)\left(2\right)\left(3\right)$
are satisfied, $\left\{ x_{i}\right\} _{i=1}^{n_{kv}}$ are linearly
independent using proof by contradiction. Assume that $\left\{ x_{i}\right\} _{i=1}^{n_{kv}}$
are linearly dependent, $\exists\left\{ \alpha_{i}\in\mathbb{R}\right\} _{i=1}^{n_{kv}}$,
not all zeros such that

\begin{align}
\overrightarrow{0} & =\sum_{i=1}^{n_{kv}}\alpha_{i}x_{i}=\sum_{i=1}^{n_{kv}}\alpha_{i}\mathcal{F}\left(q\odot k_{i}\right)\nonumber \\
 & =\sum_{i=1}^{n_{kv}}\alpha_{i}\left(a^{f}\odot\left(q\odot k_{i}\right)\right)\nonumber \\
 & =\left(a^{f}\odot q\right)\odot\left(\sum_{i=1}^{n_{kv}}\alpha_{i}k_{i}\right)\label{eq:linear_dep}
\end{align}
As $\left(2\right)\left(3\right)$ hold true, Eq. \ref{eq:linear_dep}
is equivalent to

\[
\overrightarrow{0}=\sum_{i=1}^{n_{kv}}\alpha_{i}k_{i}
\]
which contradicts $\left(1\right)$. 
\end{proof}
Prop. \ref{prop:opa_assoc} is useful as it points out the potential
of our OPA formulation for accurate associative retrieval over several
key-value pairs. That is, despite that many items are extracted to
form the relational representation, we have the chance to reconstruct
any items perfectly if the task requires item memory. As later we
use neural networks to generate $k$ and $q$, the model can learn
to satisfy conditions $\left(1\right)$ and $\left(2\right)$. Although
in practice, we use element-wise $\tanh$ to offer non-linear transformation,
which is different from $\left(3\right)$, empirical results show
that our model still excels at accurate associative retrieval. 
\begin{prop}
\label{prop:opa_assoc-1}Assume that the gates in Eq. \ref{eq:gate}
are kept constant $F_{t}=I_{t}=1$, the item memory construction is
simplified to

\[
M=\sum_{i=1}^{N+1}x_{i}\otimes x_{i},
\]
where $\left\{ x_{i}\right\} _{i=1}^{N+1}$ are positive input patterns
after feed-forward neural networks and the relational memory construction
is simplified to

\[
\mathcal{M}^{r}=\SAM_{\theta}\left(M\right),
\]
and layer normalizations are excluded, then the memory retrieval is
a two-step contraction

\[
v^{r}=\softmax\left(z^{\top}\right)\mathcal{M}^{r}f\left(x\right)
\]
\end{prop}

\begin{proof}
Without loss of generality, after seeing $N+1$ patterns $\left\{ x_{i}\right\} _{i=1}^{N+1}$,
$\SAM$ is given a (noisy or incomplete) query pattern $x$ that corresponds
to some stored pattern $x_{p}=x_{N+1}$, that is

\[
\begin{cases}
x_{p}^{\top}x\approx1\\
x_{i}^{\top}x\approx0 & i=\overline{1,N}
\end{cases}
\]
Unrolling Eq. \ref{eq:opsa} yields
\begin{align}
\SAM_{\theta}\left(M\right)\left[s\right] & =\sum_{j=1}^{n_{kv}}\mathcal{F}\left(M_{q}\left[s\right]\odot M_{k}\left[j\right]\right)\otimes M_{v}\left[j\right]\nonumber \\
 & =\sum_{j=1}^{n_{kv}}\mathcal{F}\left(W_{q}\left[s\right]\left(\sum_{i=1}^{N+1}x_{i}\otimes x_{i}\right)\right.\nonumber \\
 & \left.\odot W_{k}\left[j\right]\left(\sum_{i=1}^{N+1}x_{i}\otimes x_{i}\right)\right)\nonumber \\
 & \otimes W_{v}\left[j\right]\left(\sum_{i=1}^{N+1}x_{i}\otimes x_{i}\right)\nonumber \\
 & =\sum_{j=1}^{n_{kv}}\mathcal{F}\left(\left(\sum_{i=1}^{N}W_{q}\left[s\right]x_{i}\otimes x_{i}\right.\right.\nonumber \\
 & +\left.W_{q}\left[s\right]x_{p}\otimes x_{p}\right)\nonumber \\
 & \left.\odot\left(\sum_{i=1}^{N}W_{k}\left[j\right]x_{i}\otimes x_{i}+W_{k}\left[j\right]x_{p}\otimes x_{p}\right)\right)\nonumber \\
 & \otimes\left(\sum_{i=1}^{N}W_{v}\left[j\right]x_{i}\otimes x_{i}+W_{v}\left[j\right]x_{p}\otimes x_{p}\right)\label{eq:unroll_sam}
\end{align}
When $d>N$, it is generally possible to find $W_{q}$, $W_{k}$ and
$W_{v}$ that satisfy the following system of equations:

\[
\begin{cases}
W_{q}\left[s\right]x_{i} & =0,i=\overline{1,N},\\
W_{q}\left[s\right]x_{p} & =1\\
W_{k}\left[j\right]x_{i} & =0,i=\overline{1,N}\\
W_{k}\left[j\right]x_{p} & =1\\
W_{v}\left[j\right]x_{i} & =1,i=\overline{1,N}\\
W_{v}\left[j\right]x_{p} & =1
\end{cases}
\]
We also assume that $\mathcal{F}$ is chosen as square root function,
then Eq. \ref{eq:unroll_sam} simplifies to
\begin{align*}
\SAM_{\theta}\left(M\right)\left[s\right] & =\sum_{j=1}^{n_{kv}}\mathcal{F}\left(x_{p}\odot x_{p}\right)\otimes\sum_{i=1}^{N+1}x_{i}\\
 & =n_{kv}x_{p}\otimes\sum_{i=1}^{N+1}x_{i}\\
 & =n_{kv}\sum_{i=1}^{N+1}x_{p}\otimes x_{i}
\end{align*}

The first contraction $\softmax\left(z^{\top}\right)\mathcal{M}^{r}$
can be interpreted as an attention to $\left\{ \SAM_{\theta}\left(M\right)\left[s\right]\right\} _{s=1}^{n_{q}}$,
which equals

\[
n_{kv}\sum_{i=1}^{N+1}x_{p}\otimes x_{i}
\]
The second contraction is similar to a normal associative memory retrieval.
When we choose $f\left(x\right)=\frac{x}{n_{kv}}$, the retrieval
reads

\begin{align*}
v^{r} & =\left(n_{kv}\sum_{i=1}^{N+1}x_{p}\otimes x_{i}\right)\frac{x}{n_{kv}}\\
 & =\sum_{i=1}^{N+1}\left(x_{i}^{\top}x\right)x_{p}\\
 & \approx x_{p}
\end{align*}
\end{proof}

\subsection{Implementation of gate functions \label{subsec:Implementation-of-gate}}

\[
F_{t}\left(\mathcal{M}_{t-1}^{i},x_{t}\right)=W_{F}x_{t}+U_{F}\tanh\left(\mathcal{M}_{t-1}^{i}\right)+b_{F}
\]
\[
I_{t}\left(\mathcal{M}_{t-1}^{i},x_{t}\right)=W_{I}x_{t}+U_{I}\tanh\left(\mathcal{M}_{t-1}^{i}\right)+b_{I}
\]
Here, $W_{F}$, $U_{F}$, $W_{I}$, $W_{I}\in\mathbb{R}^{d\times d}$
are parametric weights, $b_{F},b_{I}\in\mathbb{R}$ are biases and
$+$ is broadcasted if needed. 

\subsection{Learning curves on ablation study \label{subsec:Learning-curves-abl}}

We plot the learning curves of evaluated modes for Associative retrieval
with length 30, 50 and $N^{th}$-farthest in Fig. \ref{fig:abl_lc}.
For $N^{th}$-farthest, the last input in the sequence is treated
as the query for TPR. We keep the standard number of entities/roles
and tune TPR\footnote{https://github.com/ischlag/TPR-RNN } with different
hidden dimensions (40, 128, 256) and optimizers (Nadam and Adam).
All configurations fail to converge for the normal $N^{th}$-farthest
as shown in Fig. \ref{fig:abl_lc} (right). When we reduce the problem
size to 4 $8$-dimensional input vectors, TPR can reach perfect performance,
which indicates the problem here is more about scaling to bigger relational
reasoning contexts. 

\begin{figure*}
\begin{centering}
\includegraphics[width=0.9\textwidth]{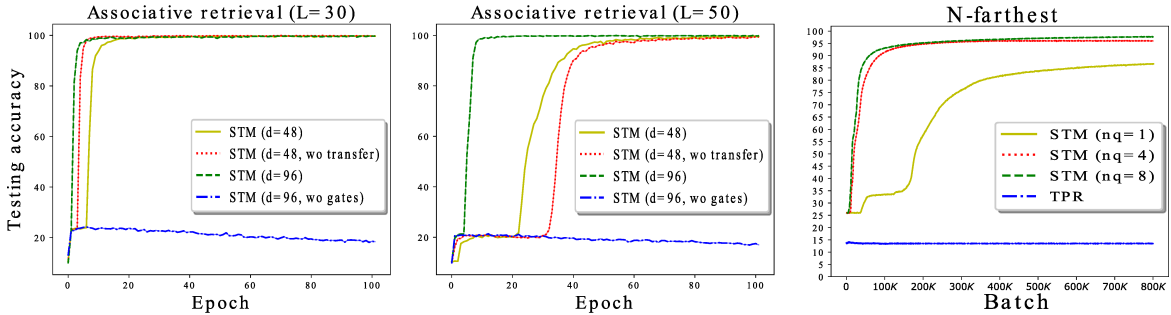}
\par\end{centering}
\caption{Testing accuracy ($\%$) on associative retrieval L=30 (left), L=50
(middle) and $N^{th}$-farthest (right).\label{fig:abl_lc}}

\end{figure*}

\subsection{Implementation of baselines for algorithmic and geometric/graph tasks
\label{subsec:Implementation-of-baselines-alg}}

Following \citet{graves2014neural}, we use RMSprop optimizer with
a learning rate of $10^{-4}$ and a batch size of 128 for all baselines.
\begin{itemize}
\item LSTM and ALSTM: Both use $512$-dimensional hidden vectors for all
tasks.
\item NTM\footnote{https://github.com/vlgiitr/ntm-pytorch}, DNC\footnote{https://github.com/deepmind/dnc}:
Both use a $256$-dimensional LSTM controller for all tasks. For algorithmic
tasks, NTM uses a $128$-slot external memory, each slot is a $32$-dimensional
vector. Following the standard setting, NTM uses 1 control head for
Copy, RAR and 5 control heads for Priority sort. For geometric/graph
tasks, DNC is equipped with $64$-dimensional $20$-slot external
memory and $4$-head controller. In geometric/graph problems, $20$
slots are about the number of points/nodes. We also tested with layer-normalized
DNC without temporal link matrix and got similar results. 
\item RMC\footnote{https://github.com/L0SG/relational-rnn-pytorch}: We
use the default setting with total 1024 dimensions for memory of 8
heads and 8 slots. We also tried with different numbers of slots $\left\{ 1,4,16\right\} $
and Adam optimizer but the performance did not change. 
\item STM: We use the same setting across tasks $n_{q}=8$, $d=96$, $n_{r}=96$.
$\alpha_{1}$,$\alpha_{2}$, and $\alpha_{3}$ are learnable. 
\end{itemize}

\subsection{Order of relationship\label{subsec:Order-of-relationship}}

In this paper, we do not formally define the concept of order of relationship.
Rather, we describe it using concrete examples. When a problem requires
to compute the relationship between items, we regard it as a first-order
relational problem. For example, sorting is first-order relational.
Copy is even zero-order relational since it can be solved without
considering item relationships. When a problem requires to compute
the relationship between relationships of items, we regard it as a
second-order relational problem and so on.

From this observation, we hypothesize that the computational complexity
of a problem roughly corresponds to the order of relationship in the
problem. For example, if a problem requires a solution whose computational
complexity between $O\left(N\right)$ and $O\left(N^{2}\right)$ where
$N$ is the input size, it means the solution basically computes the
relationship between any pair of input items and thus corresponds
to first-order relationship. Table \ref{tab:-Order-of} summarizes
our hypothesis on the order of relationship in some of our problems. 

By design, our proposed STM stores a mixture of relationships between
items in a relational memory, which approximately corresponds to a
maximum of second-order relational capacity. The distillation process
in STM transforms the relational memory to the output and thus determines
the order of relationship that STM can offer. We can measure the degree
that STM involves in relational mining by analyzing the learned weight
$\mathcal{G}_{2}$ of the distillation process. Intuitively, a high-rank
transformation $\mathcal{G}_{2}$ can capture more relational information
from the relational memory. Trivial low-rank $\mathcal{G}$ corresponds
to item-based retrieval without much relational mining (Prop. \ref{prop:SAMbili}).
The numerical rank of a matrix $A$ is defined as $r\left(A\right)=\left\Vert A\right\Vert _{F}^{2}/\left\Vert A\right\Vert _{2}^{2}$,
which relaxes the exact notion of rank \cite{rudelson2007sampling}. 

We report the numerical rank of learned $\mathcal{G}_{2}\in\mathbb{R}^{6144\times96}$
for different tasks in Table \ref{tab:-nurank}. For each task, we
run the training 5 times and take the mean and std. of $r\left(\mathcal{G}_{2}\right)$.
The rank is generally higher for tasks that have higher orders of
relationship. That said, the model tends to overuse its relational
capacity. Even for the zero-order Copy task, the rank for the distillation
transformation is still very high. 

\begin{table}
\begin{centering}
\begin{tabular}{ccc}
\hline 
{\small{}Task} & {\small{}General complexity} & {\small{}Order}\tabularnewline
\hline 
{\small{}Copy/Associative retrieval} & {\small{}$O\left(N\right)$} & {\small{}0}\tabularnewline
{\small{}Sort} & {\small{}$O\left(N\log N\right)$} & {\small{}1}\tabularnewline
{\small{}Convex hull } & {\small{}$O\left(N\log N\right)$} & {\small{}1}\tabularnewline
{\small{}Shortest path}\tablefootnote{{\small{}The input is sequence of triplets, which is equivalent to
sequence of edges. Hence, the complexity is based on the number of
edges in the graph.}} & {\small{}$O\left(E\log V\right)$} & {\small{}1}\tabularnewline
{\small{}Minimum spanning tree} & {\small{}$O\left(E\log V\right)$} & {\small{}1}\tabularnewline
{\small{}RAR/$N^{th}$-Farthest} & {\small{}$O\left(N^{2}\log N\right)$} & {\small{}2}\tabularnewline
{\small{}Traveling salesman problem} & {\small{}NP-hard} & {\small{}many}\tabularnewline
\hline 
\end{tabular}
\par\end{centering}
\caption{ Order of relationship in some problems. \label{tab:-Order-of}}
\end{table}
\begin{table}
\begin{centering}
\begin{tabular}{cc}
\hline 
Task & $r\left(\mathcal{G}_{2}\right)$\tabularnewline
\hline 
Associative retrieval & 9.42$\pm$0.5\tabularnewline
$N^{th}$-Farthest & 83.20$\pm$0.2\tabularnewline
Copy & 79.00$\pm$0.3\tabularnewline
Sort & 79.58$\pm$0.1\tabularnewline
RAR & 83.30$\pm$0.2\tabularnewline
Convex hull  & 80.78$\pm$0.6\tabularnewline
Traveling salesman problem & 83.58$\pm$0.3\tabularnewline
Shortest path & 79.81$\pm$0.2\tabularnewline
Minimum spanning tree & 79.57$\pm$0.5\tabularnewline
\hline 
\end{tabular}
\par\end{centering}
\caption{ Mean and std. of numerical rank of the leanred weight $\mathcal{G}_{2}$
for several tasks. The upper bound for the rank is 96. \label{tab:-nurank}}
\end{table}

\subsection{Geometry and graph task description\label{subsec:Geometry-and-graph-1}}

In this testbed, we use RMSprop optimizer with a learning rate of
$10^{-4}$ and a batch size of 128 for all baselines. STM uses the
same setting across tasks $n_{q}=8$, $d=96$, $n_{r}=96$. The random
one-hot features can be extended to binary features, which is much
harder and will be investigated in our future works. 

\paragraph{Convex hull }

Given a set of $N$ points with 2D coordinates, the model is trained
to output a list of points that forms a convex hull sorted by coordinates.
Training is done with $N\sim\left[5,20\right]$. Testing is done with
$N=5$ and $N=10$ (no prebuilt dataset available for $N=20$). The
output is a sequence of 20-dimensional one-hot vectors representing
the features of the solution points in the convex-hull.

\paragraph{Traveling salesman problem }

Given a set of $N$ points with 2D coordinates, the model is trained
to output a list of points that forms a closed tour sorted by coordinates.
Training is done with $N\sim\left[5,10\right]$. Testing is done with
$N=5$ and $N=10$. The output is a sequence of 20-dimensional one-hot
vectors representing the features of the solution points in the optimal
tour.

\paragraph{Shortest path }

The graph is generated according to the following rules: (1) choose
the number of nodes $N\sim\left[5,20\right]$, (2) after constructing
a path that goes through every node in the graph (to make the graph
connected), determine randomly the edge between nodes (number of edges
$E\sim\left[6,30\right]$), (3) for each edge set the weight $w\sim\left[1,10\right]$.
We generate 100,000 and 10,000 graphs for training and testing, respectively.
The representation for an input graph is a sequence of triplets followed
by 2 feature vectors representing the source and destination node.
The output is a sequence of 40-dimensional one-hot feature vectors
representing the solution nodes in the shortest path.

\paragraph{Minimum spanning tree }

We use the same generated input graphs from the Shortest path task.
The representation for an input graph is only a sequence of triplets.
The output is a sequence of 40-dimensional one-hot feature vectors
representing the features of the nodes in the solution edges of the
minimum spanning tree.

Some generated samples of the four tasks are visualized in Fig. \ref{fig:graph_samples}.
Learning curves are given in Fig. \ref{fig:graph_samples-1}.

\begin{figure*}
\begin{centering}
\includegraphics[width=0.9\textwidth]{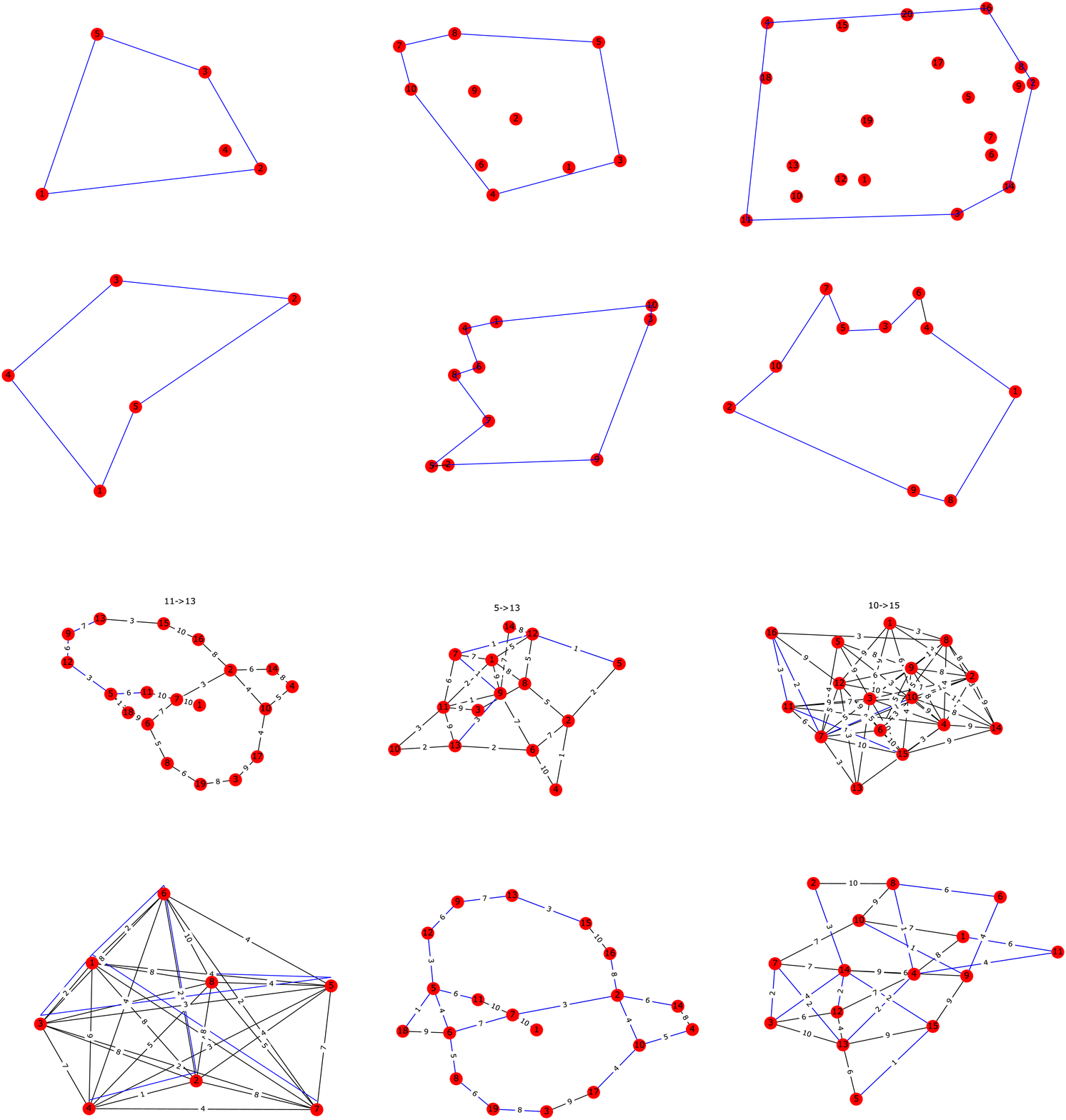}
\par\end{centering}
\caption{Samples of geometry and graph tasks. From top to bottom: Convex hull,
TSP, Shortest path and Minimum spanning tree. Blue denotes the ground-truth
solution.\label{fig:graph_samples}}
\end{figure*}
\begin{figure*}
\begin{centering}
\includegraphics[width=1\textwidth]{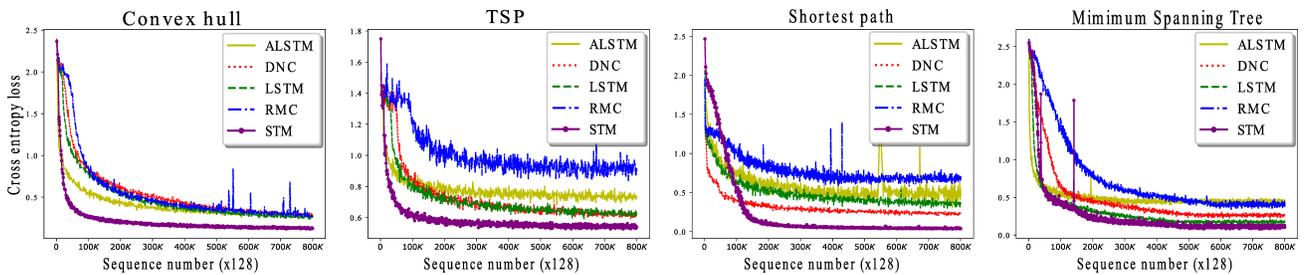}
\par\end{centering}
\caption{Learning curves on geometry and graph tasks.\label{fig:graph_samples-1}}
\end{figure*}

\subsection{Reinforcement learning task description \label{subsec:Reinforcement-learning-task}}

We trained Openai Gym's PongNoFrameskip-v4 using Asynchronous Advantage
Actor-Critic (A3C) with hyper-parameters: 32 workers, shared Adam
optimizer with a learning rate of $10^{-4}$, $\gamma=0.99$. To extract
scene features for LSTM and STM, we use 4 convolutional layers (32
kernels with $5\times5$ kernel sizes and a stride of 1), each of
which is followed by a $2\times2$ max-pooling layer, resulting in
1024-dimensional feature vectors. The LSTM 's hidden size is 512.
STM uses $n_{q}=8$, $d=96$, $n_{r}=96$.

\subsection{bAbI task description \label{subsec:babi_task}}

We use the train/validation/test split introduced in bAbI's en-valid-10k
v1.2 dataset. To make STM suitable for question answering task, each
story is preprocessed into a sentence-level sequence, which is fed
into our STM as the input sequence. The question, which is only 1
sentence, is preprocessed to a query vector. Then, we utilize the
Inference module, which takes the query as input to extract the output
answer from our relational memory $\mathcal{M}^{r}$. The preprocessing
and the Inference module are the same as in \citet{schlag2018learning}.
STM's hyper-parameters are fixed to $n_{q}=20$, $d=90$, $n_{r}=96$.
We train our model jointly for 20 tasks with a batch size of $128$,
using Adam optimizer with a learning rate of $0.006$, $\beta_{1}=0.9$
and $\beta_{2}=0.99$. Details of all runs are listed in Table \ref{tab:Results-from-10}.

\begin{table*}
\begin{centering}
\begin{tabular}{cccccccccccc}
\hline 
Task & run-1 & run-2 & run-3 & run-4 & run-5 & run-6 & run-7 & \textbf{run-8} & run-9 & run-10 & Mean\tabularnewline
\hline 
1 & 0.0 & 0.0 & 0.0 & 0.0 & 0.0 & 0.0 & 0.0 & 0.0 & 0.0 & 0.0 & 0.00 $\pm$ 0.00\tabularnewline
2 & 0.1 & 0.6 & 0.1 & 0.1 & 0.7 & 0.2 & 0.2 & 0.0 & 0.1 & 0.0 & 0.21 $\pm$ 0.23\tabularnewline
3 & 3.4 & 3.2 & 1.0 & 1.3 & 2.4 & 3.8 & 3.2 & 0.5 & 0.9 & 1.6 & 2.13 $\pm$ 1.14\tabularnewline
4 & 0.0 & 0.0 & 0.0 & 0.0 & 0.0 & 0.0 & 0.0 & 0.0 & 0.0 & 0.0 & 0.00 $\pm$ 0.00\tabularnewline
5 & 0.6 & 0.2 & 0.6 & 0.6 & 0.7 & 0.5 & 0.9 & 0.5 & 0.7 & 0.4 & 0.57 $\pm$ 0.18\tabularnewline
6 & 0.0 & 0.0 & 0.0 & 0.0 & 0.0 & 0.1 & 0.0 & 0.0 & 0.0 & 0.0 & 0.00 $\pm$ 0.00\tabularnewline
7 & 1.0 & 0.9 & 0.5 & 0.6 & 0.9 & 1.4 & 1.0 & 0.6 & 0.5 & 0.7 & 0.81 $\pm$ 0.27\tabularnewline
8 & 0.2 & 0.1 & 0.1 & 0.2 & 0.0 & 0.0 & 0.1 & 0.2 & 0.1 & 0.2 & 0.12 $\pm$ 0.07\tabularnewline
9 & 0.0 & 0.0 & 0.0 & 0.0 & 0.0 & 0.0 & 0.0 & 0.0 & 0.0 & 0.0 & 0.00 $\pm$ 0.00\tabularnewline
10 & 0.0 & 0.1 & 0.0 & 0.2 & 0.0 & 0.0 & 0.0 & 0.0 & 0.0 & 0.0 & 0.03 $\pm$ 0.06\tabularnewline
11 & 0.0 & 0.0 & 0.1 & 0.0 & 0.0 & 0.0 & 0.0 & 0.0 & 0.0 & 0.0 & 0.01 $\pm$ 0.03\tabularnewline
12 & 0.1 & 0.1 & 0.1 & 0.0 & 0.0 & 0.1 & 0.0 & 0.0 & 0.0 & 0.0 & 0.04 $\pm$ 0.05\tabularnewline
13 & 0.1 & 0.0 & 0.0 & 0.0 & 0.0 & 0.0 & 0.0 & 0.0 & 0.0 & 0.0 & 0.01 $\pm$ 0.03\tabularnewline
14 & 0.1 & 0.0 & 0.1 & 0.0 & 0.1 & 0.3 & 0.0 & 0.1 & 0.5 & 0.4 & 0.16 $\pm$ 0.17\tabularnewline
15 & 0.0 & 0.0 & 0.0 & 0.0 & 0.0 & 0.0 & 0.0 & 0.0 & 0.0 & 0.0 & 0.00 $\pm$ 0.00\tabularnewline
16 & 0.3 & 0.2 & 0.2 & 0.3 & 0.1 & 0.3 & 0.6 & 0.5 & 0.3 & 0.1 & 0.29 $\pm$ 0.15\tabularnewline
17 & 0.6 & 2.6 & 0.4 & 0.4 & 0.5 & 2.1 & 3.5 & 0.5 & 0.9 & 0.3 & 1.18 $\pm$ 1.07\tabularnewline
18 & 1.0 & 0.3 & 0.2 & 0.1 & 0.4 & 0.4 & 0.2 & 0.0 & 0.1 & 0.0 & 0.27 $\pm$ 0.28\tabularnewline
19 & 4.4 & 0.3 & 0.8 & 0.0 & 8.8 & 0.4 & 0.3 & 0.1 & 4.7 & 0.8 & 2.06 $\pm$ 2.79\tabularnewline
20 & 0.0 & 0.0 & 0.0 & 0.0 & 0.0 & 0.5 & 0.0 & 0.0 & 0.0 & 0.0 & 0.00 $\pm$ 0.00\tabularnewline
\hline 
Average & 0.59 & 0.43 & 0.21 & 0.19 & 0.73 & 0.48 & 0.50 & \textbf{0.15} & 0.44 & 0.23 & 0.39 $\pm$ 0.18\tabularnewline
\hline 
Failed task & \multirow{2}{*}{0} & \multirow{2}{*}{0} & \multirow{2}{*}{0} & \multirow{2}{*}{0} & \multirow{2}{*}{1} & \multirow{2}{*}{0} & \multirow{2}{*}{0} & \multirow{2}{*}{0} & \multirow{2}{*}{0} & \multirow{2}{*}{0} & \multirow{2}{*}{0.10 $\pm$ 0.30}\tabularnewline
(\textgreater 5\%) &  &  &  &  &  &  &  &  &  &  & \tabularnewline
\hline 
\end{tabular}
\par\end{centering}
\caption{Results from 10 runs of STM on bAbI 10k. Bold denotes best run.\label{tab:Results-from-10}}

\end{table*}

\subsection{Characteristics of memory-based neural networks\label{subsec:compare_types}}

Table \ref{tab:Characteristics-of-some} compares the characteristics
of common neural networks with memory. Biological plausibility is
determined based on the design of the model. It is unlikely that human
memory employs RAM-like behaviors as in NTM, DNC, and RMC. Fixed-size
memory is inevitable for online and life-long learning, which also
reflects biological plausibility. Relational extraction and recurrent
dynamics are often required in powerful models. As shown in the table,
our proposed model exhibits all the nice features that a memory model
should have. 

\begin{table*}
\begin{centering}
\begin{tabular}{ccccc}
\hline 
\multirow{2}{*}{Model} & Fixed-size  & Relational  & Recurrent & Biologically\tabularnewline
 & memory & extraction &  dynamics & plausible\tabularnewline
\hline 
RNN, LSTM & $\checked$ & $\vartimes$ & $\checked$ & $\checked$\tabularnewline
NTM, DNC & $\checked$ & $\vartimes$ & $\checked$ & $\vartimes$\tabularnewline
RMC & $\checked$ & $\checked$ & $\checked$ & $\vartimes$\tabularnewline
Transformer & $\vartimes$ & $\checked$ & $\vartimes$ & $\vartimes$\tabularnewline
UT & $\vartimes$ & $\checked$ & $\checked$ & $\vartimes$\tabularnewline
Attentional LSTM & $\vartimes$ & $\checked$ & $\checked$ & $\vartimes$\tabularnewline
\hline 
STM & $\checked$ & $\checked$ & $\checked$ & $\checked$\tabularnewline
\hline 
\end{tabular}
\par\end{centering}
\caption{Characteristics of some neural memory models\label{tab:Characteristics-of-some}}
\end{table*}